\def\year{2020}\relax
\documentclass[letterpaper]{article} 
\usepackage{aaai20}  
\usepackage{times}  
\usepackage{helvet} 
\usepackage{courier}  
\usepackage[hyphens]{url}  
\usepackage{graphicx} 
\usepackage{graphicx}  
\usepackage{graphicx}
\usepackage{kantlipsum}
\usepackage{pdfpages}
\usepackage{amsmath}
\usepackage{amsfonts, mathrsfs}
\usepackage[linesnumbered,ruled]{algorithm2e}
\usepackage{subcaption}
\newtheorem{theorem}{Theorem}[section]
\newtheorem{corollary}[theorem]{Corollary}

\newtheorem{lemma}[theorem]{Lemma}

\newenvironment{proof}[1][Proof]{\textbf{#1.} }{\ \rule{0.5em}{0.5em}}
\setboolean{@twoside}{false}

\title{Fatigue-Aware Bandits for Dependent Click Models }
\author{ \Large \textbf{Junyu Cao,\textsuperscript{\rm 1}\thanks{Correspondence to: Junyu Cao (jycao@berkeley.edu)} Wei Sun, \textsuperscript{\rm 2}, Zuo-Jun (Max) Shen, \textsuperscript{\rm 1} Markus Ettl\textsuperscript{\rm 2}}\\ 
\textsuperscript{\rm 1} University of California, Berkeley, California 94720 \\
\textsuperscript{\rm 2} IBM Research, Yorktown Height, New York 10591
}
 
\begin{document}

\maketitle

\begin{abstract}
As recommender systems send a massive amount of content to keep users engaged, users may experience fatigue which is contributed by 
1) an overexposure to irrelevant content, 2) boredom from seeing too many
similar recommendations. To address this problem, we consider an online learning setting where a platform 
learns a policy to recommend content that takes user fatigue into account. We propose an extension of the Dependent Click Model (DCM) to describe users' behavior. We stipulate that  for each piece of content, 
its attractiveness to a user depends on its intrinsic  relevance and a discount factor which measures how many similar contents have been  shown. Users view the recommended content 
sequentially and click on the ones that they find attractive. Users may leave the platform 
at any time, and the probability of exiting  is higher when they do not like the content. Based on user's feedback, the platform  
learns  the relevance of the underlying content as well as the discounting effect due to  content fatigue. We refer to this learning task as “fatigue-aware DCM Bandit” problem. We consider two learning scenarios depending on whether the discounting effect is known. For each scenario, we propose a learning algorithm which simultaneously explores and exploits, and characterize its regret bound. 
\end{abstract}

\section{Introduction}
Recommender systems increasingly influence how users discover content. Some well-known examples include Facebook’s News Feed, movie recommendations on Netflix, Spotify’s Discover Weekly playlists, etc. To compete for users' attention and time,  platforms push out a vast amount of content when users browse their websites or mobile apps. While a user sifts through the proliferation of content, her experience could be adversely influenced by: 1) \emph{marketing fatigue} which occurs due to an overexposure to irrelevant content, and 2) \emph{content
fatigue} which refers to the boredom from seeing too many similar recommendations.

Motivated by this phenomenon, we consider an online learning setting with a ``fatigue-aware'' platform, i.e., it 
learns a policy  to select a sequence of recommendations for a user, while being conscious that both the choices
of content and their placement could influence the experience. We propose a variant of the Dependent Click Model (DCM), where a user can click on multiple items\footnote{We use ``content'' and ``items'' interchangeably in this work.} that she finds attractive. 
For each piece of content, its attractiveness depends on its intrinsic relevance and a discount factor that
measures how many similar recommendations have already  been  shown to this user. The user may leave the platform
at any time, and the probability of exiting is higher when she does not like the content, reflecting
the presence of marketing fatigue. Meanwhile, showing too much relevant but similar content could
also reduce its attractiveness due to content fatigue and drive bored users to abandon the platform.
The platform whose objective is to maximize the total number of clicks needs
to learn users’ latent preferences in order to determine the optimal sequence of recommendations, based on users' feedback. We
refer to this online learning task which the platform faces as \emph{fatigue-aware DCM bandits}.


Our contribution of our work is fourfold. Firstly, we propose a novel model which captures the effects of marketing
fatigue and content fatigue on users’ behavior. In the original DCM model \cite{guo2009efficient,katariya2016dcm}, users can only
exit the platform upon clicking on a recommendation. 
In reality, 
users may leave at any time, and they are more likely to exit when they are not engaged with the recommendations. 
We extend the DCM model to capture marketing fatigue 
by incorporating exiting behavior after both non-clicks and clicks. To reflect content fatigue, we incorporate a discount factor into DCM such that it penalizes repetitive content of similar types and promotes diversity in its recommendations. 
Secondly, even in the offline setting where all the information is
known, the optimization problem to select a sequence of recommendation which is central to the learning task is combinatorial in nature without a straightforward efficient
algorithm. We propose a polynomial-time algorithm for this problem. Thirdly, for the online setting,
we first consider a scenario where the discount factor is known (e.g., when it can be estimated from historical data). We propose a learning algorithm
and quantify the regret bound. Lastly, we consider a more general scenario where both the discount
factor and intrinsic relevance of items need to be learned. This is a significantly more challenging 
setting to analyze because 1) we only observe partial feedback on item’s attractiveness which depends on both
the discount factor and the relevance; 2) the discount factor is dependent on the content
that we recommend. For this scenario, we also establish a regret bound for the algorithm that we propose.

Our paper is organized as follows. In Section~\ref{S.literature}, we review related work. In Section~\ref{S.formulation}, we introduce the fatigue-aware DCM model and provide an offline algorithm to find the optimal sequence when all parameters are known. In Section~\ref{S.online} and \ref{S.discount}, we introduce our learning problems for two scenarios depending on whether the discount factor is known and analyze the regret associated with the proposed algorithms. In Section~\ref{S.numerical}, we perform several numerical experiments to evaluate the performance of our algorithms.

\section{Related Literature}\label{S.literature}

Cascade Model \cite{chuklin2015click,craswell2008experimental} is a popular model that describes user's single-click behavior. Several variants of this model have been considered in the bandit literature \cite{combes2015learning,kveton2015cascading,kveton2015combinatorial}. One of its limitations  is that the positions of the items do not influence the reward since a user is expected to browse the list of recommendations until she clicks on the content that she likes. DCM \cite{guo2009efficient} generalizes Cascade Model to allow multiple clicks by incorporating a parameter to indicate the probability that a user resumes browsing after clicking on a recommendation. On the other hand, if a user does not click, she views the next item with certainty unless the sequence runs out. 
\cite{katariya2016dcm} analyzed one type of DCM bandit settings. However, 
the reward in  \cite{katariya2016dcm} 
does not exactly correspond to the number of clicks.  
In our DCM bandit setting, the rewards exactly correspond to 
the number of users' clicks, which is an ubiqutious measure on recommenders' effectiveness. In addition, as our model allows users to exit at any time, the position of the items  matters as users may exit the platform early if they do not like what has been shown initially.

Alternative multi-click models include the position-based model (PBM) \cite{richardson2007predicting}. In PBM, the click probability of an item is a product of a static position-based examination probability and the item's intrinsic relevance score.  \cite{komiyama2017position,lagree2016multiple} have investigated PBM bandits. In \cite{lagree2016multiple}, the position-based examination probability is assumed known, whereas  this quantity has to be learned together with item attractiveness in \cite{komiyama2017position}. 
In particular, \cite{komiyama2017position} derive a parameter-dependent regret bound, which assumes that the gap between parameters is strictly positive. 
While our model also includes a position-dependent parameter which discounts the attractiveness of an item, this quantity does not merely depend on the position. The discount factor depends on the content of the recommendations, more specifically, the order and the content types. For our setting, we derive a parameter-independent regret bound.

Our work also bears some resemblance to the stochastic rank-1 bandits \cite{katariya2016stochastic,katariya2017bernoulli}, where 
a row arm and a column arm are pulled simultaneously in each around, and the reward  corresponds to the product of the two values. 
In our setting, item attractiveness is a product of the discount factor and item's intrinsic relevance.  However, in rank-1 bandits, one is only interested in determining the particular pair that yields the highest reward, whereas we need to learn and rank all items, which is significantly harder.

An important topic of recommender systems is fatigue control \cite{kapoor2015just,ma2016user}. One way to manage content fatigue is to introduce diversity in recommendations \cite{radlinski2008learning,ziegler2005improving}. The linear submodular function has been used to tackle the diversification problem in bandit setting \cite{yu2016linear,yue2011linear}.
\cite{warlop2018fighting} propose a reinforcement learning framework that uses a linear reward structure which captures 
the effect of the recent recommendations on user’s preferences. 
However, instead of recommending specific items as in our work,  \cite{warlop2018fighting} restrict to recommending genres or content types to a user. While \cite{cao2019dynamic,wang2019thompson} have also studied marketing fatigue in an online setting, their setting only captures a single click, whereas our model allows multiple clicks and incorporates the effect of content fatigue in addition to marketing fatigue.

\section{Problem Formulation}\label{S.formulation}
In this section, we formally introduce the fatigue-aware DCM model and present an algorithm to determine the optimal sequence of recommendations in an offline setting where all underlying parameters are known.

\subsection{Setting}
Suppose there are $N$ available items  (e.g., songs, videos, articles) for the platform to choose from, denoted as $[N]$. Each item belongs to one of the $K$ types (e.g., genres, categories, topics). Let ${\bf S}=(S_1,\cdots, S_m)$ denote a sequence of recommended items. For each item $i$, $u_i$ denotes its intrinsic relevance to a user, and $C(i)=j$ if item $i$ belongs to type $j$. We define $h_i({\bf S})$ as the number of items with type $C(i)$ shown  before item $i$. Similar to \cite{warlop2018fighting}, we model the impact of showing too many similar contents as a discount factor  on items' attractiveness. 
More precisely, define the attractiveness of item $i$ as $z_i({\bf S}):=f(h_i({\bf S}))u_i$,  where $f(\cdot)$ represents the discount factor on users' preferences due to content fatigue, which depends on how many items of the same type have been shown. $f(\cdot)$ is assumed to be a decreasing function. Without loss of generality, we assume $f(0)=1$.

Given a list of recommendations ${\bf S}=(S_1,\cdots, S_m)$, a user examines it sequentially, starting from $S_1$. 
The user clicks on item $i$ when she finds it attractive. The platform receives a  reward for every click. DCM models the multi-click behavior by incorporating a parameter $g$ which is the probability that the user will see more recommendations after a click. 
In case of no click or skip (unlike in the original DCM \cite{guo2009efficient,katariya2016dcm} where a user examines the next item with probability 1), we use $q$ to denote the probability that the user will resume browsing. As users are more likely to stay on the platform and continue browsing after consuming a piece of good content\footnote{We provide some evidence of  this phenomenon based on the estimates from a real dataset, with  more details  in Section \ref{S.numerical}.}, we let $q\leq g$ 
. 
The interaction between a user and recommended content is illustrated in Fig~\ref{fig:flowchart}.

\begin{figure}[h!]
\centering
  \includegraphics[width=\linewidth]{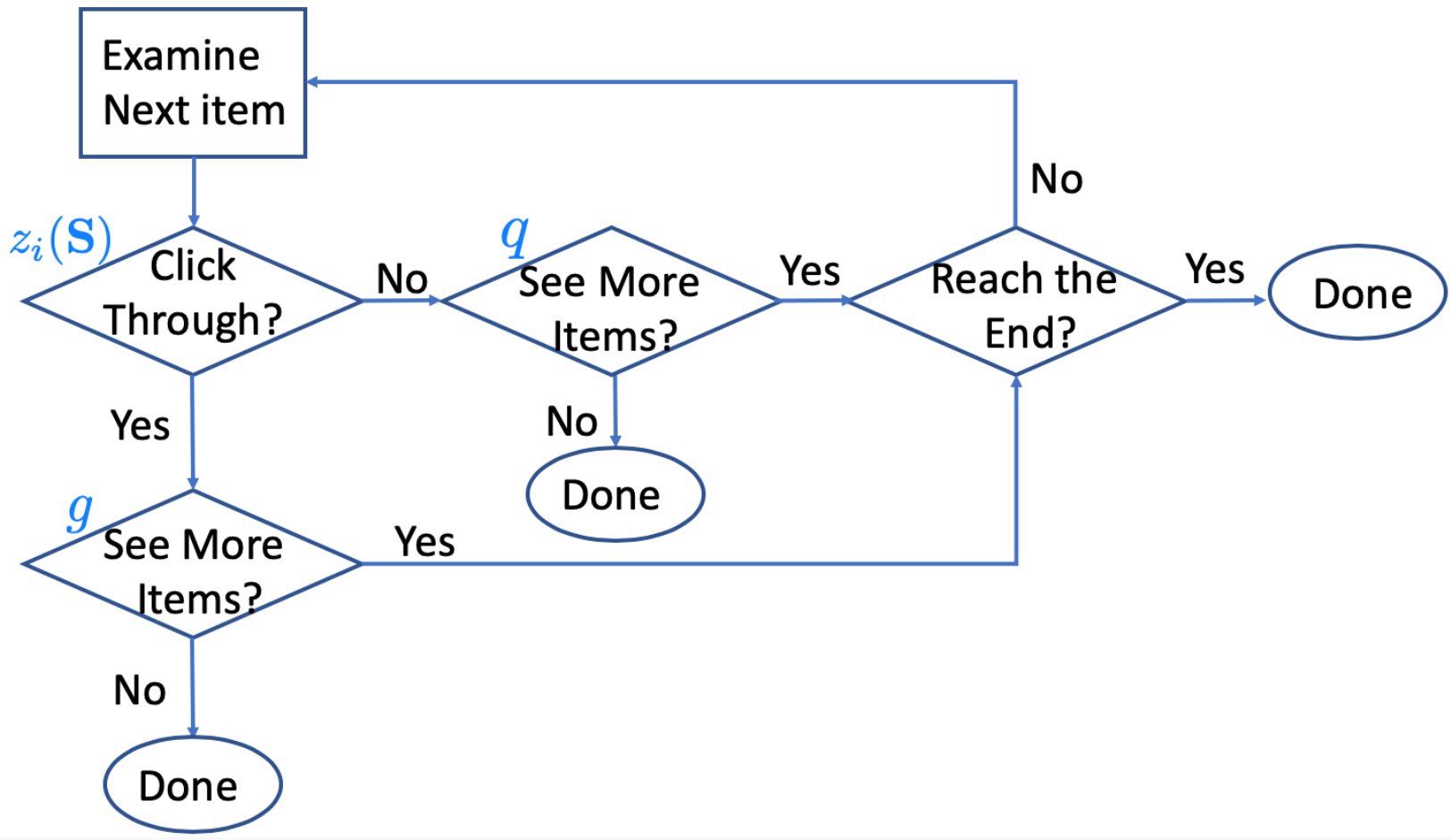}
  \caption{User behavior in a fatigue-aware DCM.}
  \label{fig:flowchart}
\end{figure}


The probability of clicking on item $i$ from a sequence $\bold{S}$, denoted as $\mathbb{P}_i(\bold{S})$, depends on its position in the list, its own content, as well as the content of other items shown previously. Formally,
\begin{align*}
 &\mathbb{P}_{i}(\bold{S})=\\
&\left\{
\begin{aligned}
&u_i,& \text{ if } i\in S_1\\
&\prod_{k=1}^{l-1}\left(gz_{I(k)}({\bf S})+q(1-z_{I(k)}({\bf S}))\right)z_{i}({\bf S})& \text{ if } i\in S_l
\end{aligned}
\right.\notag
\end{align*} 
 where $I(\cdot)$ denotes the index function, i.e., $I(k)=i$ if and only if $S_k=\{i\}$. When $i$ is the first item of the sequence, the probability of clicking is simply $u_i$, which is its intrinsic relevance.  For the remainder of the sequence, the probability of clicking on $i$ as the $l^{th}$ item is 
the joint probability that 1) the user finds item $i$ attractive after taking content fatigue into account, $z_{i}({\bf S})$; 2) she remains on the platform after examining the first $l-1$ items, $\prod_{k=1}^{l-1}\left(gz_{I(k)}({\bf S})+(1-z_{I(k)}({\bf S}))q\right)$, which  accounts  for both clicking and skipping behavior. 


\subsection{Platform's optimization problem}
The platform's objective is to maximize the total number of clicks by optimizing the sequence of recommended items. We use $R(\bold{S})$ to denote the platform's expected reward, i.e., $E[R(\bold{S})]=\sum_{i \in [N]} \mathbb{P}_{i}(\bold{S}).$ Thus, the platform's optimization problem can be written as 
\begin{align}
\max_{\bold{S}}\quad&  E[R(\bold{S})]\label{eq:optimization}\\
s.t. \quad & S_i\cap S_j=\emptyset, \forall i\neq j.\nonumber
\end{align}
The constraint requires  no duplicated items in the recommendations. We define ${\bf S^*}$ as the optimal sequence of content, and $R^*$ as the corresponding maximum reward. While this problem is  combinatorial in nature, Theorem \ref{T.oneproduct} shows that this problem is polynomial-time solvable and  the optimal sequence ${\bf S^*}$ can be identified by Algorithm~\ref{A.optimalsequence}. 

\begin{algorithm}[h!]
  \For{$i=1:K$}{
  Sort $u_{j}$ for $C(j)=i$ in an decreasing order $o(\cdot)$ where $o(j)=r$ if item $j$ is at the $(r+1)^{th}$ position\;
  Set $\lambda_j=u_{j}f(o(j))$\;
   }
   
  Sort $\lambda_j$ for all  $j =1:N$ in an decreasing order $o'(\cdot)$ where $o'(r)=j$ if item $j$ is at the $r^{th}$ position\;
  Set ${\bf S}=(o'(1),o'(2),\cdots,o'(N)).$
 \caption{Determine the optimal sequence ${\bf S^*}$ to the offline combinatorial problem }\label{A.optimalsequence}
\end{algorithm}

\begin{theorem}\label{T.oneproduct}
Algorithm~\ref{A.optimalsequence} finds the optimal sequence ${\bf S^*}$.
 \end{theorem}

Due to the space limit, we present the proof outline in the main paper and the detailed proofs  are included in the Supplementary Material.\\
{\bf Proof outline:} We prove the result by contradiction. If the optimal sequence $\bold{S}^*$ is not ordered by Algorithm~\ref{A.optimalsequence}, then there exists a neighboring pair $I(i)$ and $I(i+1)$ such that either 1) $I(i)$ and $I(i+1)$ belong to the same type and $u_{I(i)}<u_{I(i+1)}$ or 2) $I(i)$ and $I(i+1)$ belong to different types and $f(h_{I(i)}({\bf S}^*))u_{I(i)}<f(h_{I(i+1)}({\bf S}^*))u_{I(i+1)}$. We then show that swapping items $I(i)$ and $I(i+1)$ increases the expected reward, which is a contradiction to the statement that ${\bf S}^*$ is optimal. $\Box$

For each category, Algorithm~\ref{A.optimalsequence} first ranks items based on their intrinsic relevance. If an item has the $(r+1)^{th}$ largest  relevance score within a category, we will then multiply the score with $f(r)$ to compute its attractiveness. Next, the algorithm ranks items across all categories in their decreasing attractiveness. The complexity of Algorithm~\ref{A.optimalsequence} is $O(N\log N)$. We also observe that 
the optimal sequence is completely determined by items' intrinsic relevance ${\bf u}$ and the discount factor $f$, and is independent of the resuming probabilities $g$ and $q$. In other words, the platform only needs to learn ${\bf u}$ and $f$ for the online learning task. Nevertheless, the resuming probabilities influence the learning speed as we will investigate in the following sections. 
\section{Learning with Known  Discount Factor $f$}\label{S.online}
In the previous section, we assume all the parameters  are  known to the platform. It is natural to ask what the platform should do in the absence of such knowledge. Beginning from this section, we will present learning algorithms to the fatigue-aware DCM bandit problem and characterize the corresponding regret bound. We first investigate a scenario where the discounting effect  caused by content fatigue is known (e.g., it could be estimated from historical data), and the platform needs to learn the item's intrinsic relevance  ${\bf u}$ to determine the optimal sequence. We want to emphasize the differences between this setting and other {learning to rank} problems such as \cite{combes2015learning,kveton2015cascading,kveton2015combinatorial}. Firstly, in our setting, only partial feedback to a list is observed as the user might leave before viewing all the recommendations. Secondly, as the order of the content influences their attractiveness, the click rate of an item is non-stationary. Thirdly, in many previous ``learning to rank" related problems, the position-biased examine probability is usually assumed to be fixed for each position. However, in our setting, this examine probability depends on the content of previously shown recommendations. 

Assume users arrive at time $t=1,\cdots, T$. For each user $t$, the platform determines a sequence of items $\tilde{\bold{S}}^t$. 
We evaluate a learning algorithm by its expected cumulative rewards under policy $\pi$, or equivalently, its expected cumulative regret  which is defined as 
$Regret_\pi(T;{\bf u})=E_\pi\left[\sum_{t=1}^T R(\bold{S}^*,{\bf u})-R(\tilde{\bold{S}}^t,{\bf u})\right].$ 

\subsection{Algorithm FA-DCM-P}\label{sect:FA-DCM-P}

Our strategy, which we call {F}atigue-{A}ware DCM with Partial information (FA-DCM-P), is a UCB-based policy. 
We first define an unbiased estimator for ${\bf u}$ and its upper confidence bound. 

We define the pair $(i,j)$ as an event that  item $j$ is the ${(i+1)}^{th}$ recommendation from the category $C(j)$. That is, $i$ recommendations from the same category $C(j)$ have been shown before item $j$.  We define $\mathcal{T}_{ij}$ as the epoch for the event $(i,j)$, i.e., $t\in \mathcal{T}_{ij}$ if and only if user $t$ observes the event $(i,j)$. Let $T_{ij}(t)$ denote the number of times that the event $(i,j)$ occurs by time $t$, i.e., $T_{ij}(t)=|\mathcal{T}_{ij}(t)|$. Define $T_j(t)=\sum_i T_{ij}(t)$. Let $z_{ij}^t$ be the binary feedback for user $t$ observing the event $(i,j)$, where $z_{ij}^t=1$ indicates that the user clicks,  and $z_{ij}^t=0$ means no click or skip.
%
%
Define $\hat{u}_{j,t}$ as the estimator for $u_j$, and $u_{j,t}^{UCB}$ as its upper confidence bound, i.e., 
\begin{equation}\label{E.update}
\hat{u}_{j,t} = \frac{1}{T_j(t)}\sum_i\sum_{t\in \mathcal{T}_{ij}}\frac{z_{ij}^t}{f_i}, \quad\quad u_{j,t}^{UCB}=\hat{u}_{j,t}+\sqrt{2\frac{\log t}{T_j(t)}}.
\end{equation}

In Algorithm FA-DCM-P, for a user arriving at time $t$, we use $\bold{u}_{t-1}^{UCB}$  to calculate the current optimal sequence of recommendations. Define $w_t$ as the last item examined by user $t$, which occurs when one of the following feedback is observed: 1) the user exits after clicking on $w_t$; 2) the user who views $w_t$  exits without clicking; 3) the sequence runs out. When the last examined item $w_t$ is shown,  we update $T_j(t)$ and the upper confidence bound to $\bold{u}_{t}^{UCB}$ based on the feedback.  

 \begin{algorithm}
 \textbf{Initialization:} Set $u_{i,0}^{UCB}=1$ and $T_i(0)=0$ for all $i\in [N]$; $t=1$\;
 \While{$t<T$}{
  Compute 
$\tilde{\bold{S}}^t=\arg\max_{\bold{S}}\quad E[R(\bold{S},\bold{u}_{t-1}^{UCB})]$ according to Theorem~\ref{T.oneproduct}\;
  Offer sequence $\tilde{\bold{S}}^t$, observe the user's feedback\;
 {
   update $ T_i(t)$; update ${\bf u}_t^{UCB}$ according to Equation~\eqref{E.update}; $t = t+1$\;
    }
 }
 \caption{[FA-DCM-P] An algorithm for fatigue-aware DCM bandit when $f$ is known}\label{A.fknown}
\end{algorithm}

\subsection{Regret analysis for Algorithm FA-DCM-P}
In this section, we quantify the regret bound for our proposed algorithm. We first derive the following lemmas which will be used in the regret analysis.

 \begin{lemma}\label{L.largedeviation}
 For any $t$ and $j\in [N]$ we have
 $$P\left(u_{j,t}^{UCB}-\sqrt{8\frac{\log t}{T_j(t)}}<u_j<u_{j,t}^{UCB}\right)\geq 1- \frac{2}{t^4}.$$  
 \end{lemma}
 
  \begin{lemma}\label{L.compare}
Assume $\bold{S}^*$ is the optimal sequence of messages.  Under the condition that $0\leq \bold{u} \leq \bold{u}'$, we have
$$ E[R(\bold{S}^*,{\bf u}')]\geq E[R(\bold{S}^*,{\bf u})].$$
\end{lemma}
\begin{proof}[Proof of Lemma~\ref{L.compare}]
We prove this theorem by induction. For the optimal sequence $\bold{S}^*=(S_1,S_2,\cdots, S_N)$, since $u_{I(N)}'\geq u_{I(N)}$, we have $E[R(S_{N},\bold{u}')]\geq E[R(S_N,\bold{u})]$. Assume the inequality 
$$E[R((S_{k+1},\cdots, S_{N}),\bold{u}')]\geq E[R((S_{k+1},\cdots, S_{N}),\bold{u})]$$ holds, and now we prove that $E[R((S_k,\cdots, S_{N}),\bold{u}')]\geq E[R((S_k,\cdots, S_{N}),\bold{u})]$. Since 
\begin{align*}
    & E[R((S_k,\cdots, S_{N}),\bold{u}')]\\
    &= u_{I(k)}'+(gu_{I(k)}'+q(1-u_{I(k)}')) E[R((S_k,\cdots, S_{N}),\bold{u}')]\\
    &=u_{I(k)}'+((g-q)u_{I(k)}'+q) E[R((S_k,\cdots, S_{N}),\bold{u}')]\\
    &\geq u_{I(k)}+((g-q)u_{I(k)}+q) E[R((S_k,\cdots, S_{N}),\bold{u})]\\
    &=E[R((S_k,\cdots, S_{N}),\bold{u})].
\end{align*}
Therefore, we reach the desired result.
\end{proof}

\begin{theorem}\label{T.regret}
The regret of Algorithm FA-DCM-P during time $T$ is bounded by 
$$Regret_\pi(T)\leq  C \left(\frac{1-g^N}{1-g}\right)^{3/2}\sqrt{NT\log T}.$$
for some constant $C$.
\end{theorem}
The detailed proof is included in the Supplementary Material.\\
{\bf Proof outline:}  Define event $A_{i,t}=\{u_{i,t}^{UCB}-\sqrt{8\frac{\log t}{T_i(t)}}<u<u_{i,t}^{UCB}\}$ and
$E_t=\bigcap_{i=1}^N A_{i,t}.$ On the ``large probability" event $E_t$, with Lemma~\ref{L.compare}, we obtain  $E[R(\tilde{\bold{S}}^t,{\bf u})]\leq E[R(\bold{S}^{*},\bold{u})]\leq E[R(\bold{S}^*,\bold{u}^{UCB})]\leq E[R(\tilde{\bold{S}}^t,\bold{u}^{UCB})]$. It implies that $E[R(\bold{S}^{*},\bold{u})]-E[R(\tilde{\bold{S}}^t,{\bf u})]\leq E[R(\tilde{\bold{S}}^t,\bold{u}^{UCB})]-E[R(\tilde{\bold{S}}^t,{\bf u})]$. Let $I_t(\cdot)$ denote the index function for user $t$. We then show that the cumulative difference on event $E_t$ can be bounded from above by
\begin{align*}
&E_\pi\left[\sum_{t=1}^T \sum_{i=1}^{|\tilde{{\bf S}}^t|}\prod_{k=1}^{i-1}((g-q)f(h_{I_t(k)}(\tilde{{\bf S}}^t))u^{UCB}_{I_t(k),t}+q)\right.\\
&\left.f(h_{I_t(i)}(\tilde{{\bf S}}^t))u^{UCB}_{I_t(i),t}1(E_t)\right]-E_\pi\left[\sum_{t=1}^T \sum_{i=1}^{|\tilde{{\bf S}}^t|}\prod_{k=1}^{i-1}((g-q)\right.\\
&\left.f(h_{I_t(k)}(\tilde{{\bf S}}^t))u_{I_t(k)}+q)f(h_{I_t(i)}(\tilde{{\bf S}}^t))u_{I_t(i)}1(E_t)\right]\\
&\leq \frac{1-g^N}{1-g}E_\pi\left[\sum_{t=1}^T \sum_{j=1}^N \kappa_{j,t}(u_{j,t}^{UCB}-u_{j,t})1(E_t)\right],
\end{align*}
where $\kappa_{j,t}=\prod_{k=1}^{l_t(j)-1}((g-q)f(h_{I_t(k)}(\tilde{{\bf S}}^t))u_{I_t(k)}+q)$. It denotes the probability that user $t$ observes item $j$  and  $l_t(i)$ specifies the position of item $i$ in $\tilde{\bold{S}}^t$. 
We then show that the regret term can be further bounded by $\frac{1-g^N}{1-g}\sqrt{2\log T}\sum_{i\in[N]}\sqrt{E_\pi[T_i(T)]}$. Since $\sum_{i\in[N]}E_\pi[T_i(T)]\leq \frac{1-g^N}{1-g}T$, we can bound it by $C_1 (\frac{1-g^N}{1-g})^{3/2}\sqrt{NT\log T}.$ To complete the proof, we show that on the ``small probability" event $E_t^c$, 
the regret can also be bounded. $\Box$

The following results can be shown under two special cases of the DCM setting.

\begin{corollary}\label{C.cascade}
When $g=q=1$, the regret can be bounded by
$$Regret_\pi(T)\leq  CN^2\sqrt{T\log T}.$$
\end{corollary}
\begin{corollary}\label{C.L_limit}
When at most $L$ items can be recommended to a single user, the regret can be bounded by
$$Regret_\pi(T)\leq C L^{3/2}\sqrt{NT\log T}.$$
\end{corollary}

{Corollary~\ref{C.cascade} characterizes the regret bound under a special setting where users never exit the platform and browse all the recommendations.} In theorem \ref{T.regret}, we do not limit the length of the sequence as online recommenders continuously send content to users to keep them on the platforms. If at most $L$ items will be recommended, the correponding regret is shown in Corollary~\ref{C.L_limit}. 
The detailed proofs can be found in the Supplementary Material.

\section{Learning with Unknown Discount Factor $f$}\label{S.discount}
In this section, we consider a more general scenario where the discount factor $f$ is also unknown and needs to be learned, in conjunction with learning items' intrinsic relevance scores $\bold{u}$.  
Before delving into our approach, we first discuss a few alternative approaches and their trade-offs. A straightforward approach towards this problem is to treat each combination of $f(i)$ and $u_j$ as an arm whose expected reward is $z_{ij}:=f(i)u_j$. However, it is not clear how to solve the offline combinatorial problem. 
An alternative approach is to  view the problem as a  generalized linear bandit which can be solved by GLM-UCB \cite{li2017provably}. Taking logarithm, we have $\log(z_{ij})=\log(f(i))+\log(u_j)$, which is a linear model. This approach is problematic for the following reasons: a) When $u_j$ and $f(i)$ go to 0, $\log(u_j)$ and $\log(f(i))$ can go to negative infinity, which implies that the parameter space is unbounded; b) In each step, GLM-UCB needs to compute the maximum likelihood estimators, which is computationally costly.
We now present our approach which is a UCB-based algorithm for this learning task.




\subsection{Algorithm FA-DCM}
Throughout this section, we define $M$ as a threshold such that after showing $M$ items from the same category, the discounting effect on item attractiveness due to content fatigue stays the same. That is, $f(r) = f(r+1)$ for $r\geq M$. Note that the maximum possible value for $M$ could be $N$.



Following the notations introduced in Section \ref{sect:FA-DCM-P}, we define the unbiased estimator for $f_i$ and $u_j$ as  $$\hat{f}^t_i = \frac{1}{\hat{T}_i(t)}\sum_j\sum_{r\in \mathcal{T}_{ij}(t)}\frac{z_{ij}^r}{\hat{u}_{j,t}},\quad \text{and} \quad \hat{u}_{j,t}=\frac{\sum_{r\in\mathcal{T}_{0j}(t)}z^r_{0j}}{T_{0j}(t)},$$
where $\hat{T}_i(t)=\sum_j T_{ij}(t)$, the subscript in $T_{0j}$ and $z^r_{0j}$  represents the event $(0,j)$, that is, item $j$ is the first message from category $C(j)$ to be shown. 
The corresponding upper confidence bounds at time $t$ are defined as
\begin{equation}\label{E.updatef}
f_{t}^{UCB}(i)=\hat{f}_i^t +\Delta_i^t,\quad \text{ and }\quad  u_{j,t}^{UCB}=\hat{u}_{j,t}+\sqrt{2\frac{\log t}{T_{0j}(t)}}, 
\end{equation} 
where
\begin{align*}
\Delta_{i}^t=\sum_j\sum_{r\in \mathcal{T}_{ij}(t)}\frac{z_{ij}^r}{\hat{T}_i(t)}\frac{1}{(\hat{u}_{j,t})^2}\left(1-\frac{1}{\hat{u}_{j,t}}\sqrt{\frac{\log t}{T_{0j} (t)}}\right)\\
1\left(\hat{u}_{j,t}\geq \sqrt{\frac{\log t}{T_{0j}(t)}}\right)\sqrt{\frac{\log t}{T_{0j}(t)}}+\sqrt{\frac{\log t}{\hat{T}_i(t)}}.
\end{align*}
$\Delta_i^t$ consists of two parts: the first part refers to the exploration term related to $\mu_j$ for all $j\in [N]$, and the second part is the exploration term with respect to $f(i)$.

We propose a learning algorithm, which we refer to as FA-DCM for the scenario when both the function $f$ and $\bold{u}$ need to be learned. Define $\theta^{UCB}=(f^{UCB},\bold{u}^{UCB})$ and $\theta=(f,\bold{u})$. At time $t$, we recommend a sequence $\tilde{\bold{S}}^t$ based on  $\theta_{t-1}^{UCB}$. Suppose there exists an item $i$ which we have not collected sufficient feedback such that $T_i(t)<\alpha T^{2/3}$ where $\alpha$ is a tuning parameter. This item will then be inserted to the beginning of the sequence $\tilde{\bold{S}}^t$ to guarantee that it will be examined. This procedure ensures that we obtain a reliable estimate for $u_i$, since the discount factor $f(0)=1$ for the first item.  Once the feedback of the user at time $t$ is observed, we then update $f^{UCB}_t$ and $\bold{u}^{UCB}_t$. Finally, it may be possible that some estimated values of $f_{t}^{UCB}(i)$ violate the decreasing property of $f$, i.e.,  $f_{t}^{UCB}(i)>f_{t}^{UCB}(i-1)$, then we need to correct them to enforce the property.


 \begin{algorithm}
 \textbf{Initialization:} Set $u_{j,0}^{UCB}=1$ and $f^{UCB}(i)=1$ for all $j\in [N]$ and $i\in [M]$ ; $t=1$\;
 \While{$t<T$}{
  Compute 
$\tilde{\bold{S}}^t= argmax_{\bold{S}} \quad E[R(\bold{S},\theta_{t-1}^{UCB})]$ according to Theorem~\ref{T.oneproduct}\;
\If{there exists $i\in[N]$ such that $T_i(t)<\alpha T^{2/3}$}{
	Insert item $i$ to the beginning of of $\tilde{{\bf S}}^t$ as the first item\;
}
  Offer sequence $\tilde{\bold{S}}^t$, observe the user's feedback\;
 {
   Update $u_{j,t}^{UCB}$ for all $j$ and $f_{t}^{UCB}(i)$ for all $i$ according to Equation~\eqref{E.updatef}\;
   \For{$i=1:M$}{
   \If{$f_{t}^{UCB}(i)>f_{t}^{UCB}(i-1)$}{
   $f_{t}^{UCB}(i)=f_{t}^{UCB}(i-1)$;
   }
   }
  $t = t+1$\;
    }
 }
 \caption{[FA-DCM] An algorithm for fatigue-aware DCM when $f$ is unknown}\label{A.DCMf}
\end{algorithm}

\subsection{Regret analysis for Algorithm FA-DCM}
Define $V_{i}^t = \frac{1}{\hat{T}_i}\sum_j\sum_{t\in \mathcal{T}_{ij}(t)}\frac{z_{ij}^t}{\mu_j}$.
We first have the following lemma to bound the difference between $V_i^t$ and $f_i$.

\begin{lemma}\label{L.boundVf} For any $t$, 
$P(|V_i^t-f(i)|\geq \sqrt{\frac{\log t}{\hat{T}_i(t)}})\leq\frac{2}{t^2}$.
\end{lemma}

\begin{lemma}\label{L.boundf} For any $1\leq i\leq M$, we have
$\sum_{t=1}^T P(|f_{t}^{UCB}(i)-  f(i)|\geq \Delta_i^t )\leq CN,$
for some constant $C$, where $f_{t}^{UCB}(i)$ is defined in Equation~\eqref{E.updatef}.
\end{lemma}
The detailed proof is included in the Supplementary Material.\\
{\bf Proof outline:}
First we note that 
\begin{align}
&P\left(|f_{t}^{UCB}(i)- f(i)|\geq \Delta_i^t \right)\notag\\
&\leq P\left(\left|\frac{1}{\hat{T}_i(t)}\sum_j\sum_{r\in \mathcal{T}_{ij}}z_{ij}^r\left(\frac{1}{\hat{u}_{j,t}}-\frac{1}{u_j}\right)\right|\right.\geq\notag\\
& \hspace{5mm} \left.\sum_j \sum_{r\in \hat{\mathcal{T}}_{ij}}\frac{z_{ij}^r}{\hat{T}_i(t)}\frac{1}{\hat{u}_{j,t}^2}\left(1-\frac{1}{\hat{u}_{j,t}}\sqrt{\frac{\log t}{T_{0j} (t)}}\right)\right.\notag\\
&\hspace{5mm}\left.1\left(\hat{u}_{j,t}\geq \sqrt{\frac{\log t}{T_{0j}(t)}}\right)\sqrt{\frac{\log t}{T_{0j}(t)}}\right) \label{E.part1}\\
&\hspace{5mm}+ P\left(\left|\frac{1}{\hat{T}_i(t)}\sum_j\sum_{r\in \mathcal{T}_{ij}}\frac{z_{ij}^r}{u_j}-f(i)\right|\geq \sqrt{\frac{\log t}{T_i(t)}}\right) \label{E.part2}
\end{align}
For Equation~\eqref{E.part1}, we can bound it above by
$$\sum_j \frac{2}{t^2} + \exp\left(-\frac{1}{2}T_{0j}(t)u_j^2\right) + P\left(\sqrt{\frac{\log t}{T_{0j}(t)}}<\frac{u_j}{2}\right).$$
For Equation~\eqref{E.part2}, using Lemma~\ref{L.boundVf}, we have
\begin{align*}
P\left(\left|\frac{1}{\hat{T}_i(t)}\sum_j\sum_{r\in \mathcal{T}_{ij}}\frac{z_{ij}^r}{u_j}-f(i)\right|\geq \sqrt{\frac{\ln t}{T_i(t)}}\right)\leq \frac{2}{t^2}.
\end{align*}
It implies that 
\begin{align*}
&\sum_{t=1}^T P(|f_{t}^{UCB}(i)-  f(i)|\geq \Delta_i^t )\leq CN
\end{align*}
for some constant $C$. $\Box$

\begin{lemma}\label{L.compare2}
Assume $\bold{S}^*$ is the optimal sequence of messages.  Under the condition that $0\leq \bold{u} \leq \bold{u}'$ and $0\leq f\leq f'$, we have
$$ E[U(\bold{S}^*, f',\bold{u}')]\geq E[U(\bold{S}^*,f,\bold{u})].$$
\end{lemma}
{\bf Proof: }Similar to the proof of Lemma~\ref{L.compare}, we can prove it by induction. $\Box$




\begin{theorem}\label{T.regretf}
The regret of Algorithm FA-DCM during time $T$ is bounded by 
$$Regret_\pi(T)\leq C \left(\frac{1-g^N}{1-g}\right)^2\sqrt{NT^{4/3}\log T}.$$
\end{theorem}
{\bf Proof outline:} Define the following events,   $A_{j,t}=\{u_{j,t}^{UCB}-\sqrt{8\frac{\log t}{T_j(t)}}<u_j<u_{j,t}^{UCB}\}$, $B_{i,t}=\{f_{t}^{UCB}(i)-2\Delta_i^t<f(i)<f_{t}^{UCB}(i)\},$
and 
$E_t=\bigcap_{j=1}^N (A_{j,t})\bigcap_{i=1}^M( B_{i,t})$. 
Let $\tilde{f}_{t}^{UCB}(i)=\min_{j\leq i} f_{t}^{UCB}(j)$,  $\tilde{B}_{i,t}=\left\{\tilde{f}_{t}^{UCB}(i)-2\Delta_i^t<f(i)<\tilde{f}_{t}^{UCB}(i)\right\}$ and $\tilde{E}_t=\bigcap_{j=1}^N (A_{j,t})\bigcap_{i=1}^M( \tilde{B}_{i,t}).$ Note that if $E_t$ holds, $\tilde{E}_t$ also holds. We consider time $t$ is in the exploration phase, denoted as $t\in \mathcal{E}$, if there exists $j\in[N]$ such that $T_{0j}(t)<\alpha t^{2/3}$. Similar to the proof for Theorem~\ref{T.regret}, 
the regret quantity $E[\sum_{t\notin\mathcal{E}} R(\bold{S}^*,{\bf u})-R(\tilde{\bold{S}}^t,{\bf u})]$ on event $\tilde{E}_t$ for $t\notin \mathcal{E}$ can be first bounded from above by
\begin{align*}
 &\frac{1-g^N}{1-g}\left(E_\pi\left[\sum_{t\notin \mathcal{E}} \sum_{j=1}^N \kappa_{j,t}(u_{j,t}^{UCB}-u_{j})1(\tilde{E}_t)\right]\right.\\
 &+E_\pi\left.\left[\sum_{t\notin\mathcal{E}} \sum_{i=1}^M \tilde{\kappa}_{i,t}(\tilde{f}_t^{UCB}(i)-f(i))1(\tilde{E}_t)\right]\right),
 \end{align*}
where 
$\tilde{\kappa}_{i,t}=\sum_{\tilde{l}_t(i)}\prod_{k=1}^{\tilde{l}_t(i)-1}((g-q)f(h_{I_t(k)}(\tilde{{\bf S}}^t))u_{I_t(k)}+q)$ denotes the probability of observing an item as the $i^{th}$ recommendation within its own category. The regret term can be further bounded by $C \left(\frac{1-g^N}{1-g}\right)^2\sqrt{NT^{4/3}\log T}$.
Finally using Lemma~\ref{L.boundf}, we can bound the ``small probability" event $E_t^c$ and the cumulative regret on $t\in\mathcal{E}$. $\Box$

\section{Numerical Experiments}\label{S.numerical}
In this section, we perform four sets of numerical experiments to evaluate the performance of our online learning algorithms. In the first two experiments, we investigate the robustness of Algorithm FA-DCM-P and FA-DCM respectively. In the last two experiments, we compare our algorithm with a benchmark using   simulated and  real datasets respectively.

\subsection{Robustness study}
{\bf Experiment I: With a known discount factor} In this experiment, we investigate the robustness of Algorithm FA-DCM-P, when the discount factor $f$ is known to the platform. We consider a setting with three categories, and each contains 10 items. Items' intrinsic relevance $\bf{u}$ is uniformly generated from $[0,0.5]$. The known discount factor is set to $f(r)=\exp(-0.1*(r-1))$. We fix the resuming probability after non-clicks to $q=0.7$, and compare three  cases by varying the resuming probabilities after clicks, i.e, Case 1: $g=0.95$; Case 2: $g=0.85$; Case 3: $g=0.75$. For each case, we run 20 independent simulations.

{\bf Experiment II: With an unknown discount factor} We assume both $f$ and $\bf{u}$ are unknown in this experiment, and study the robustness of Algorithm FA-DCM. There are three categories which contain 10 items each. We consider three cases, where we use $q=0.7$ as in Experiment I and vary  $p$ and the discount factor $f$. More precisely, 
Case 4: $g=0.85$, $f(r)=\exp(-0.1(r-1))$; 
Case 5: $g=0.85$,  $f(r)=\exp(-0.15(r-1))$; 
Case 6: $g=0.75$, $f(r)=\exp(-0.1(r-1))$.

{\bf Result: }The left plot in Fig~\ref{fig:Algorithm2} shows the average regret for Algorithm FA-DCM-P as the solid line and its 95\% sampled confidence interval as the shaded area, generated from 20 independent simulations for each case.  Average values for the regret at $T=10,000$  are 
307.52, 277.77 and 265.73 
for Case 1, 2 and 3 respectively. The result shows that  the regret decreases as $g$ decreases, which is consistent with the regret bound shown in Theorem~\ref{T.regret}. 

 The right plot in Fig~\ref{fig:Algorithm2} shows the results for Algorithm FA-DCM, where the average regret at $T=10,000$ is 1237.82, 1178.24, and 991.97 for Case 4, 5 and 6 respectively. Comparing Case 4 and 5, we find that when the discount factor increases, the regret decreases. Meanwhile, comparing Case 4 and 6, it shows that the regret increases with $g$, which is consistent with the conclusion in Theorem~\ref{T.regretf}. Comparing Case 2 and 4, as well as Case 3 and 6, the simulation results show that the regret is much larger when $f$ also needs to be learned.


\begin{figure}[h!]
\centering
\begin{subfigure}{.24\textwidth}
  \centering
  \includegraphics[width=\linewidth]{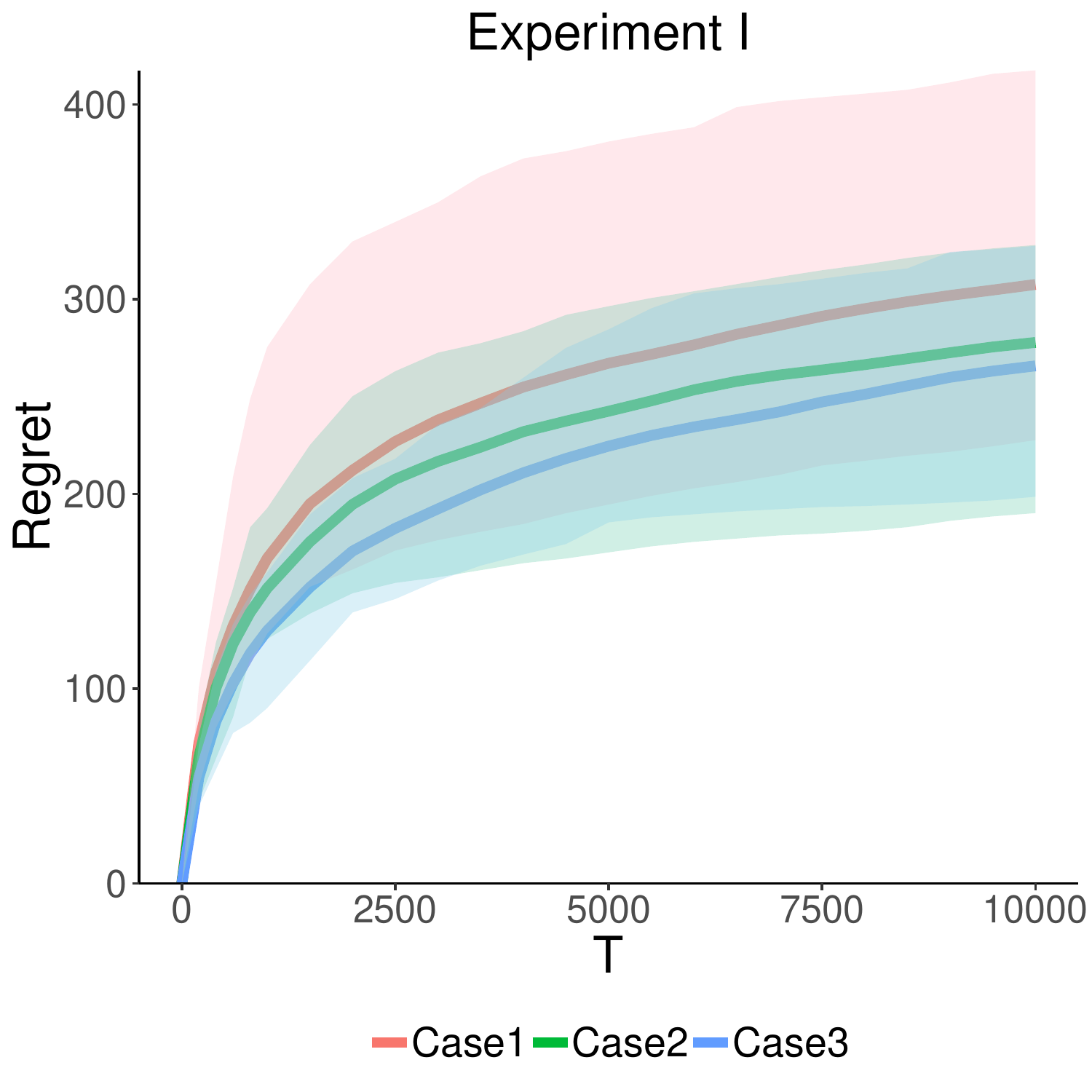}
\end{subfigure}%
\begin{subfigure}{.24\textwidth}
  \centering
  \includegraphics[width=\linewidth]{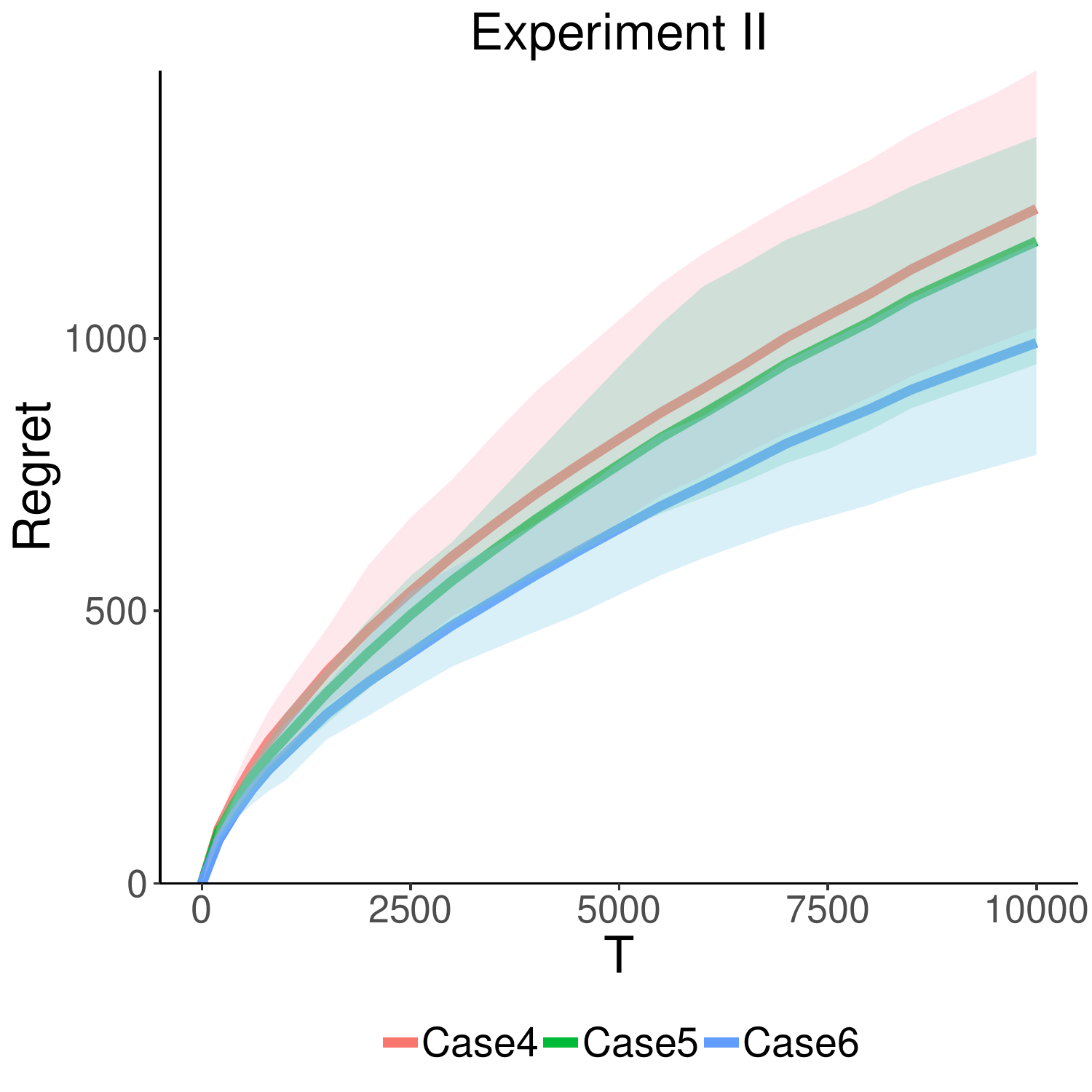}
\end{subfigure}

\caption{Simulation results for Experiments I and II }\label{fig:Algorithm2}
\end{figure}

\subsection{Comparison with a benchmark algorithm}
Since there is no other algorithm addressing our specific setting, we choose the explore-then-exploit algorithm as a benchmark for Algorithm FA-DCM. During the exploration time steps, items are randomly selected. For the first $t$ time steps, we ensure that there are $\beta\log t$ periods for exploration where $\beta$ is a tuning parameter.  That is, if the number of exploration periods before $t$ is less than $\beta\log t$, then we choose $t$ as the exploration period where items are randomly selected. 
During the exploitation period, we use the empirical mean of $\bold{u}$ and $f$ to find the optimal sequence. In both Experiment III and IV, we run 20 independent simulations to compare the performance of our algorithm against the benchmark.

{\bf Experiment III: With simulated data} 
We set parameters $g=0.75$, $q=0.7$ and $f(r)=\exp(-0.1(r-1))$. Items' intrinsic relevance $\bf{u}$ is uniformly generated from $[0,0.5]$.  

{\bf Experiment IV: With real data} The data is from Taobao\footnote{ \url{https://tianchi.aliyun.com/datalab/dataSet.html?spm=5176.100073.0.0.14d53ea7Rleuc9\&dataId=56}}, which contains 26 million ad display and click logs from 1,140,000 randomly sampled users from the website of Taobao for 8 days (5/6/2017-5/13/2017). To fit the data into our framework, 
we use long period of inactivity from the last interaction with the website as a proxy for exiting the platform. Based on the data, we estimate that the average  probability to resume browsing after non-clicks is 0.823, i.e., $q=0.823$. Meanwhile, the probability to resume browsing after clicking is higher, i.e., $g=0.843$. For the discount factor, we use $f(r)=\exp(-0.1(r-1))$. In this experiment, we consider five different item categories, and select the top 20 selling products from each category. We estimate the click probabilities of these 100 products, and use them as the ``groundtruth'' for $u_i$. 


{\bf Result: } Fig~\ref{fig:Algorithm3} compares the performance of our algorithm and the benchmark based on 20 independent simulation. Our algorithm outperforms the benchmark in both experiments, highlighting the benefits of simultaneous exploration and exploitation. In particular, in Experiment III which is shown as the left plot, the average regret for the benchmark is 1279.88 at T = 10,000, which is 29.02\% higher than that of Algorithm FA-DCM. Meanwhile, in Experiment IV as shown as the right figure, The average regret for Algorithm FA-DCM is 1553.41, while the regret for the benchmark is 1677.37 at $T=10,000$.

\begin{figure}[h!]
\centering
\begin{subfigure}{.24\textwidth}
  \centering
  \includegraphics[width=\linewidth]{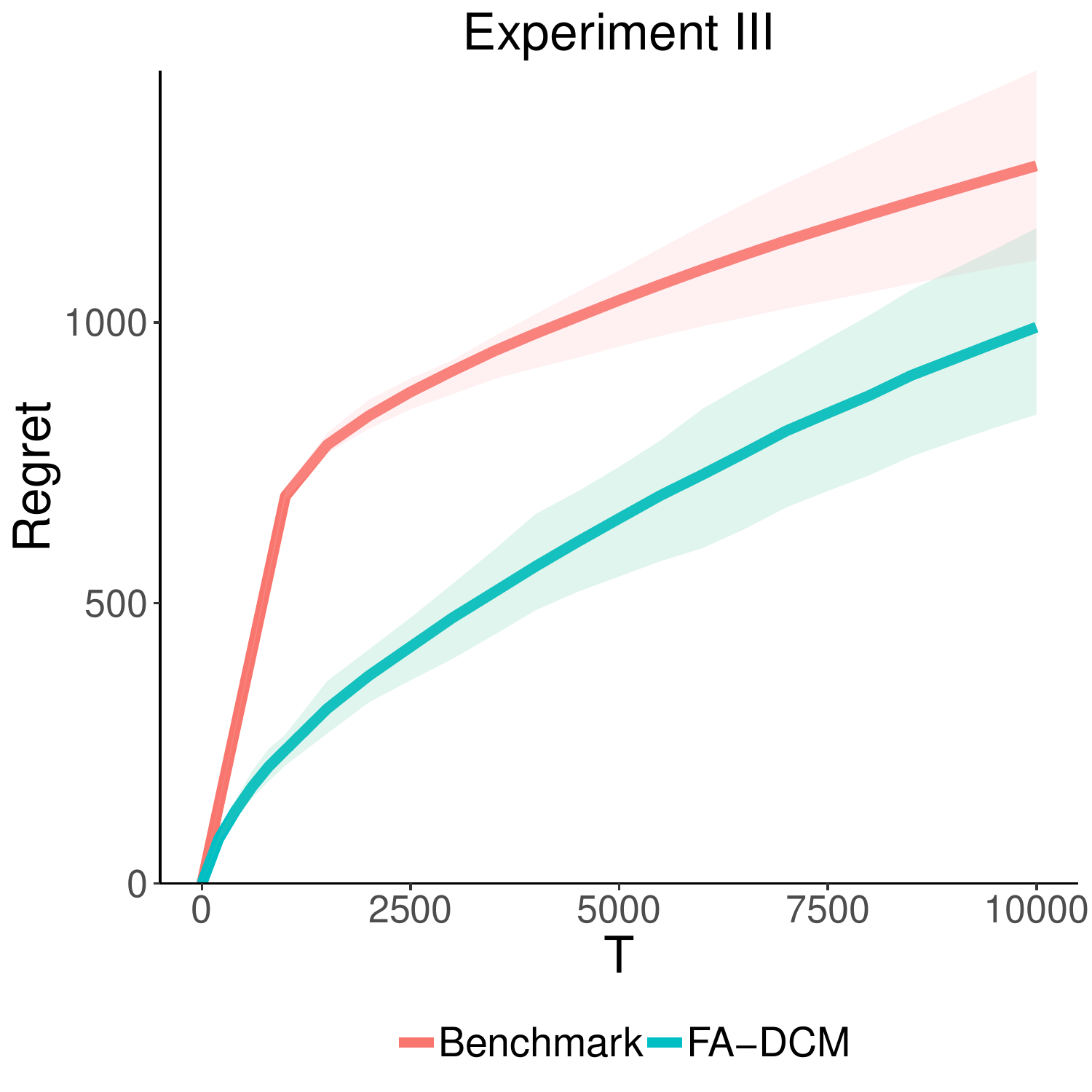}
\end{subfigure}%
\begin{subfigure}{.24\textwidth}
  \centering
  \includegraphics[width=\linewidth]{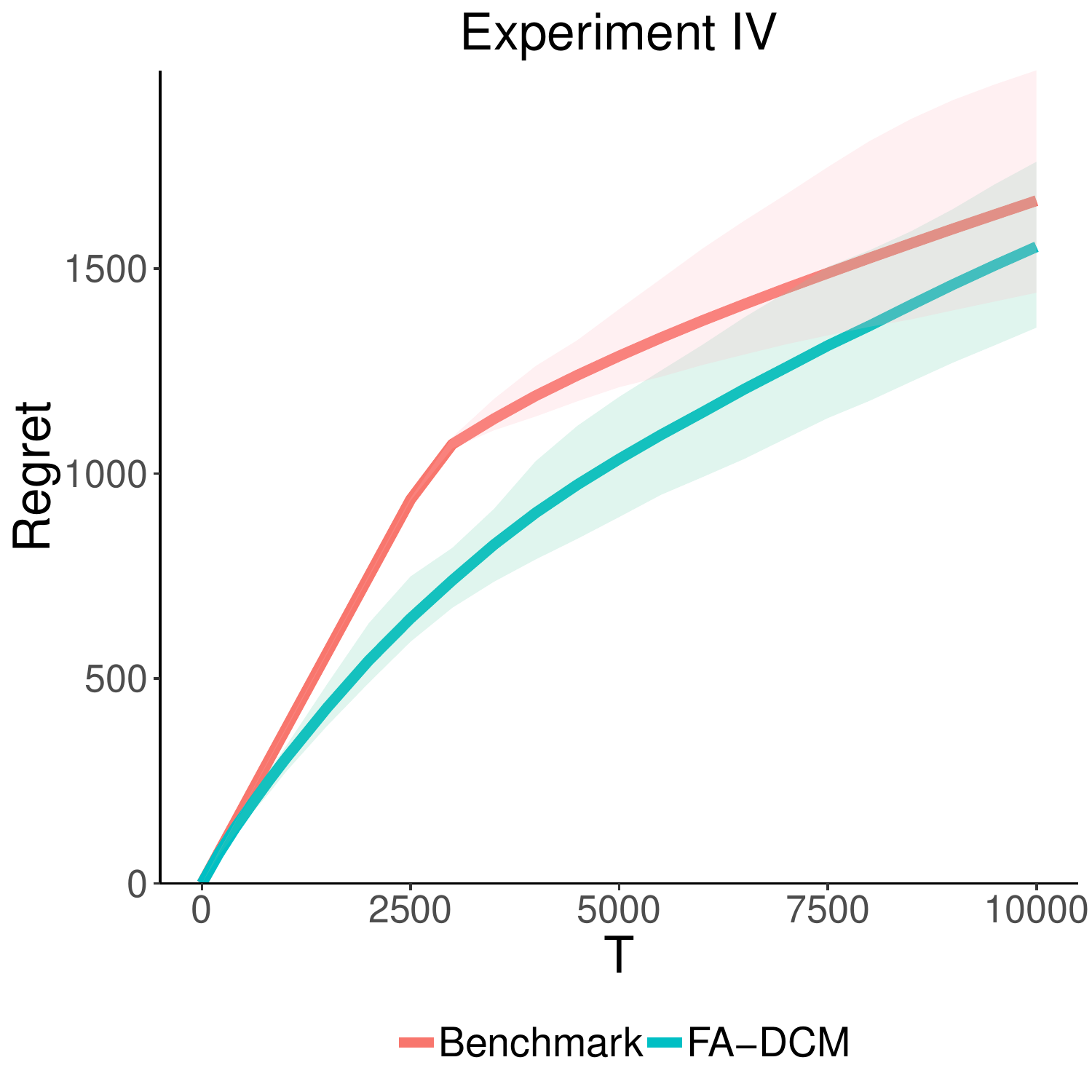}
\end{subfigure}
\caption{Simulation results for Experiments III and IV }\label{fig:Algorithm3}
\end{figure}

\section{Conclusion}
In this work, we proposed an online recommender system which incorporates both marketing fatigue and content fatigue. The setting takes into account of user's fatigue which is triggered by an overexposure to irrelevant content, as well as the boredom one might experience from seeing too much similar content. Depending on whether the discount factor $f$ is known to the platform, we proposed two online learning algorithms and quantified its regert bound for the DCM bandit setting. We further showed the regret bounds for two learning algorithms and used numerical experiments to illustrate the performance. 

There are several interesting future directions. One natural extension is to consider the contextual version of this model by incorporating item features and user attributes, such that the system is capable of producing personalized recommendations. Another direction for future work is to approach the fatigue-aware DCM bandits via Thompson Sampling, although the regret analysis remains a challenging problem. Last but not least, as users' preferences often change with time, it would be interesting to incorporate non-stationary item intrinsic relevance ($u_i$) into the model. 

\bibliography{fatigue}
\bibliographystyle{aaai}





\newpage
\newpage
~\\


\section*{Supplementary material}

# Supplementary Material

**Theorem 3.1** *Algorithm 1 finds the optimal sequence.*

**Proof of Theorem 3.1.** We prove the result by contradiction. Assume the optimal sequence of items is $\mathbf{S}$. Recall that $I$ is the index function, where $I(i) = k$ if and only if $S_i = \{k\}$. If the optimal sequence is not ordered by Algorithm 1, then there exists a neighboring pair $I(i)$ and $I(i+1)$ such that either 1) $I(i)$ and $I(i+1)$ belong to the same type and $u_{I(i)} < u_{I(i+1)}$ or 2) $I(i)$ and $I(i+1)$ belong to different types and $f(h_{I(i)}(\mathbf{S}))u_{I(i)} < f(h_{I(i+1)}(\mathbf{S}))u_{I(i+1)}$.

We first consider scenario 1, i.e., $I(i)$ and $I(i+1)$ are of the same type. For ease of notation, let $f_1 = f(h_{I(i)}(\mathbf{S}))$, $f_2 = f(h_{I(i)}(\mathbf{S}) + 1)$, and $u_1 = u_{I(i)}$ and $u_2 = u_{I(i+1)}$. Note that $f_1 \geq f_2$ and $u_1 < u_2$. Define

$$x_1 = f_1 u_1, \quad x_2 = f_2 u_2$$

and

$$x_3 = f_1 u_2, \quad x_4 = f_2 u_1.$$

Then we have

$$\begin{aligned}
&E[R((S_i, S_{i+1}, \cdots, S_N))] \\
&= x_1 + (gx_1 + q(1 - x_1)) \times (x_2 + (gx_2 + (1-x_2)q)E[R(S_{i+2}, \cdots, S_N)]) \\
&= x_1 + ((g-q)x_1 + q) \times (x_2 + ((g-q)x_2 + q)E[R(S_{i+2}, \cdots, S_N)]),
\end{aligned}$$

and

$$\begin{aligned}
&E[R((S_{i+1}, S_i, \cdots, S_N))] \\
&= x_3 + (gx_3 + q(1 - x_3)) \times (x_4 + (gx_4 + (1-x_4)q)E[R(S_{i+2}, \cdots, S_N)]) \\
&= x_3 + ((g-q)x_3 + q) \times (x_4 + ((g-q)x_4 + q)E[R(S_{i+2}, \cdots, S_N)]),
\end{aligned}$$

Define $R' = R(S_{i+2}, \cdots, S_N)$. It implies that

$$\begin{aligned}
&E[R((S_i, S_{i+1}, \cdots, S_N))] - E[R((S_{i+1}, S_i, \cdots, S_N))] \\
&= x_1 + qx_2 + q(g-q)(x_1 + x_2)E[R'] + (g-q)^2 x_1 x_2 E[R'] + (g-q)x_1 x_2 \\
&\quad - \left(x_3 + qx_4 + q(g-q)(x_3 + x_4)E[R'] + (g-q)^2 x_3 x_4 E[R'] + (g-q)x_3 x_4\right) \\
&= x_1 - x_3 + q(x_2 - x_4) + q(g-q)(x_1 - x_3 + x_2 - x_4)E[R'] \\
&= f_1(u_1 - u_2) + qf_2(u_2 - u_1) + q(g-q)(f_1(u_1 - u_2) + f_2(u_2 - u_1)) \\
&= (f_1 - qf_2)(u_1 - u_2) + q(g-q)(f_1 - f_2)(u_1 - u_2) < 0,
\end{aligned}$$

1

where the last inequality holds because $f_1 \geq f_2$ and $u_1 < u_2$. Therefore, swapping message $I(i)$ and $I(i+1)$ will increase the profit.

We now consider scenario 2), i.e., $I(i)$ and $I(i+1)$ belong to different types.

Define
$$x_1 = f(h_{I(i)}(\mathbf{S}))u_{I(i)}, \quad x_2 = f(h_{I(i+1)}(\mathbf{S}))u_{I(i+1)}.$$

Then we have
$$\begin{aligned} &E[R((S_i, S_{i+1}, \cdots, S_N))] \\ &= x_1 + (gx_1 + q(1-x_1)) \times (x_2 + (gx_2 + (1-x_2)q)E[R(S_{i+2}, \cdots, S_N)]) \\ &= x_1 + ((g-q)x_1 + q) \times (x_2 + ((g-q)x_2 + q)E[R(S_{i+2}, \cdots, S_N)]), \end{aligned}$$

and
$$\begin{aligned} &E[R((S_{i+1}, S_i, \cdots, S_N))] \\ &= x_2 + (gx_2 + q(1-x_2)) \times (x_1 + (gx_1 + (1-x_1)q)E[R(S_{i+2}, \cdots, S_N)]) \\ &= x_2 + ((g-q)x_2 + q) \times (x_1 + ((g-q)x_1 + q)E[R(S_{i+2}, \cdots, S_N)]), \end{aligned}$$

which implies that
$$\begin{aligned} &E[R((S_i, S_{i+1}, \cdots, S_N))] - E[R((S_{i+1}, S_i, \cdots, S_N))] \\ &= x_1 + qx_2 + q(g-q)(x_1+x_2)E[R'] + (g-q)^2 x_1 x_2 E[R'] \\ &\quad - \left(x_2 + qx_1 + q(g-q)(x_2+x_1)E[R'] + (g-q)^2 x_1 x_2 E[R']\right) \\ &= (1-q)(x_1 - x_2) < 0, \end{aligned}$$

Therefore, swapping message $I(i)$ and $I(i+1)$ will increase the profit. This is a contradiction to the fact that $\mathbf{S}$ is the best strategy. ∎

**Lemma 4.1** *For any $t$ and $j \in [N]$ we have*
$$P\left(u_{j,t}^{UCB} - \sqrt{8\frac{\log t}{T_j(t)}} < u_j < u_{j,t}^{UCB}\right) \geq 1 - \frac{2}{t^4}.$$

**Proof of Lemma 4.1.** We first note that for any $i$ and $t \in \mathcal{T}_{ij}$,
$$E[z_{ij}^t/f(i)] = u_j.$$

If implies that $E[\hat{u}_{j,t}] = u_j$. Recall that $\hat{u}_{j,t} = \frac{1}{T_j(t)} \sum_i \sum_{t \in \mathcal{T}_{ij}} \frac{z_{ij}^t}{f(i)}$, then for any $j \in [N]$, we have
$$\begin{aligned} &P(\hat{u}_{j,t} - E[\hat{u}_{j,t}] \geq \delta) \\ &= P(e^{s(\hat{u}_{j,t}-E[\hat{u}_{j,t}])} \geq e^{su}) \leq e^{-s\delta}E[e^{s(\hat{u}_{j,t}-E[\hat{u}_{j,t}])}] \\ &= e^{-s\delta}E\left[\exp\left(\frac{s}{T_j(t)}\left(\sum_{i=1}^M \frac{1}{f(i)}\sum_{t \in \mathcal{T}_{ij}}(z_t - f(i)u_j)\right)\right)\right] \end{aligned}$$

2

$$
\begin{aligned}
&= e^{-s\delta} \prod_{i=1}^{M} E\left[\exp\left(\frac{s}{T_j(t)f(i)} \sum_{t \in \mathcal{T}_{ij}} (z_t - f(i)u_j)\right)\right] \\
&= e^{-s\delta} \prod_{i=1}^{M} \prod_{t \in \mathcal{T}_{ij}} E\left[\exp\left(\frac{s}{T_j(t)}\left(\frac{z_t}{f(i)} - u_j\right)\right)\right] \\
&\leq e^{-s\delta} \exp\left(T_j(t)\frac{s^2}{8}\right) = \exp\left(-s\delta + \frac{s^2 T_j(t)}{8}\right).
\end{aligned}
$$

Let $s = 4\delta T_j(t)$, then we have

$$\exp\left(-s\delta + \frac{s^2 T_j(t)}{8}\right) = \exp\left(-2\delta^2 T_j(t)\right)$$

Let $\delta = \sqrt{2\log t / T_j(t)}$. Then we have

$$P\left(u_{j,t}^{UCB} - \sqrt{8\frac{\log t}{T_j(t)}} < u_j < u_{j,t}^{UCB}\right) = 1 - P\left(|\hat{u}_{j,t} - u_j| \geq \sqrt{2\frac{\log t}{T_j(t)}}\right) \geq 1 - \frac{2}{t^4}.$$

∎

**Theorem 4.1** *The regret of Algorithm FA-DCM-P during time $T$ is bounded by*

$$Regret_\pi(T) \leq C \left(\frac{1-g^N}{1-g}\right)^{3/2} \sqrt{NT \log T}$$

*for some constant $C$.*

**Proof of Theorem 4.1.** Define $\mathcal{F}_t$ as the filtration associated with the policy $\pi$ up to time $t$. Let

$$A_{i,t} = \left\{u_{i,t}^{UCB} - \sqrt{8\frac{\log t}{T_i(t)}} < u < u_{i,t}^{UCB}\right\} \quad \text{and} \quad E_t = \bigcap_{i=1}^{N} A_{i,t}.$$

Define $\tilde{\mathbf{S}}^t$ as the optimal sequence when the corresponding parameters are $\mathbf{u}_t^{UCB}$. Let $I_t(\cdot)$ be the index function of messages in $\tilde{\mathbf{S}}^t$. As a tradition, define the product $\prod_{j=r}^{k} a_j = 1$ for any value $a_j$ when $k < r$. Note that $\tilde{\mathbf{S}}^t$ is determined by $\mathcal{F}_{t-1}$. On event $E_t$, according to Lemma **??**, we have

$$E_\pi\left[R(\tilde{\mathbf{S}}^t, \mathbf{u})\right] \leq E_\pi[R(\mathbf{S}^*, \mathbf{u})] \leq E_\pi[R(\mathbf{S}^*, \mathbf{u}_t^{UCB})] \leq E_\pi[R(\tilde{\mathbf{S}}^t, \mathbf{u}_t^{UCB})],$$

which implies that on event $E_t$, we have

$$E_\pi\left[\sum_{t=1}^{T} \left(R(\mathbf{S}^*, \mathbf{u}) - R(\tilde{\mathbf{S}}^t, \mathbf{u})\right) \mathbf{1}(E_t)\right]$$

$$\leq E_\pi\left[\sum_{t=1}^{T} E_\pi\left[\left(R(\tilde{\mathbf{S}}^t, \mathbf{u}_t^{UCB}) - R(\tilde{\mathbf{S}}^t, \mathbf{u})\right) \mathbf{1}(E_t) \Big| \mathcal{F}_{t-1}\right]\right]$$

3

$$= E_\pi \left[ \sum_{t=1}^{T} \sum_{i=1}^{|\tilde{\mathbf{S}}^t|} \prod_{k=1}^{i-1} \left( (g-q) f(h_{I_t(k)}(\tilde{\mathbf{S}}^t)) u_{I_t(k),t}^{UCB} + q \right) f(h_{I_t(i)}(\tilde{\mathbf{S}}^t)) u_{I_t(i),t}^{UCB} 1(E_t) \right]$$

$$- E_\pi \left[ \sum_{t=1}^{T} \sum_{i=1}^{|\tilde{\mathbf{S}}^t|} \prod_{k=1}^{i-1} \left( (g-q) f(h_{I_t(k)}(\tilde{\mathbf{S}}^t)) u_{I_t(k)} + q \right) f(h_{I(i)}(\tilde{\mathbf{S}}^t)) u_{I_t(i)} 1(E_t) \right]. \quad (0.1)$$

We further have

$$E_\pi \left[ \sum_{t=1}^{T} \sum_{i=1}^{|\tilde{\mathbf{S}}^t|} \prod_{k=1}^{i-1} \left( (g-q) f(h_{I_t(k)}(\tilde{\mathbf{S}}^t)) u_{I_t(k),t}^{UCB} + q \right) f(h_{I_t(i)}(\tilde{\mathbf{S}}^t)) u_{I_t(i),t}^{UCB} 1(E_t) \right]$$

$$- E_\pi \left[ \sum_{t=1}^{T} \sum_{i=1}^{|\tilde{\mathbf{S}}^t|} \prod_{k=1}^{i-1} \left( (g-q) f(h_{I_t(k)}(\tilde{\mathbf{S}}^t)) u_{I(k)} + q \right) f(h_{I_t(i)}(\tilde{\mathbf{S}}^t)) u_{I_t(i)} 1(E_t) \right]$$

$$= E_\pi \left[ \sum_{t=1}^{T} \sum_{i=1}^{|\tilde{\mathbf{S}}^t|} \prod_{k=1}^{i-1} \left( (g-q) f(h_{I_t(k)}(\tilde{\mathbf{S}}^t)) u_{I_t(k),t}^{UCB} + q \right) f(h_{I_t(i)}(\tilde{\mathbf{S}}^t)) u_{I_t(i),t}^{UCB} 1(E_t) \right]$$

$$- E_\pi \left[ \sum_{t=1}^{T} \sum_{i=1}^{|\tilde{\mathbf{S}}^t|} \prod_{k=1}^{i-1} \left( (g-q) f(h_{I_t(k)}(\tilde{\mathbf{S}}^t)) u_{I_t(k)} 1(k=1) + (g-q) f(h_{I_t(i)}(\tilde{\mathbf{S}}^t)) u_{I_t(k),t}^{UCB} 1(k>1) + q \right) \right.$$
$$\left. \times \left( f(h_{I_t(i)}(\tilde{\mathbf{S}}^t)) u_{I_t(i)} 1(i=1) + f(h_{I_t(k)}(\tilde{\mathbf{S}}^t)) u_{I_t(k),t}^{UCB} 1(i>1) \right) 1(E_t) \right]$$

$$+ E_\pi \left[ \sum_{t=1}^{T} \sum_{i=1}^{|\tilde{\mathbf{S}}^t|} \prod_{k=1}^{i-1} \left( (g-q) f(h_{I_t(k)}(\tilde{\mathbf{S}}^t)) u_{I_t(k)} 1(k=1) + (g-q) f(h_{I_t(i)}(\tilde{\mathbf{S}}^t)) u_{I_t(k),t}^{UCB} 1(k>1) + q \right) \right.$$
$$\left. \times \left( f(h_{I_t(i)}(\tilde{\mathbf{S}}^t)) u_{I_t(i)} 1(k=1) + f(h_{I_t(k)}(\tilde{\mathbf{S}}^t)) u_{I_t(k),t}^{UCB} 1(k>1) \right) 1(E_t) \right]$$

$$- E_\pi \left[ \sum_{t=1}^{T} \sum_{i=1}^{|\tilde{\mathbf{S}}^t|} \prod_{k=1}^{i-1} \left( (g-q) f(h_{I_t(k)}(\tilde{\mathbf{S}}^t)) u_{I_t(k)} 1(k \leq 2) + (g-q) f(h_{I_t(k)}(\tilde{\mathbf{S}}^t)) u_{I_t(k),t}^{UCB} 1(k>2) + q \right) \right.$$
$$\left. \times \left( f(h_{I_t(i)}(\tilde{\mathbf{S}}^t)) u_{I_t(i)} 1(i \leq 2) + f(h_{I_t(k)}(\tilde{\mathbf{S}}^t)) u_{I_t(k),t}^{UCB} 1(i>2) \right) 1(E_t) \right]$$

$$+ E_\pi \left[ \sum_{t=1}^{T} \sum_{i=1}^{|\tilde{\mathbf{S}}^t|} \prod_{k=1}^{i-1} \left( (g-q) f(h_{I_t(k)}(\tilde{\mathbf{S}}^t)) u_{I_t(k)} 1(k \leq 2) + (g-q) f(h_{I_t(k)}(\tilde{\mathbf{S}}^t)) u_{I_t(k),t}^{UCB} 1(k>2) + q \right) \right.$$
$$\left. \times \left( f(h_{I_t(i)}(\tilde{\mathbf{S}}^t)) u_{I_t(i)} 1(i \leq 2) + f(h_{I_t(k)}(\tilde{\mathbf{S}}^t)) u_{I_t(k),t}^{UCB} 1(i>2) \right) \right] + \cdots$$

$$- E_\pi \left[ \sum_{t=1}^{T} \sum_{i=1}^{|\tilde{\mathbf{S}}^t|} \prod_{k=1}^{i-1} \left( (g-q) f(h_{I_t(k)}(\tilde{\mathbf{S}}^t)) u_{I_t(k)} + q \right) f(h_{I_t(i)}(\tilde{\mathbf{S}}^t)) u_{I_t(i)} 1(E_t) \right]$$

$$\leq E_\pi \left[ \sum_{t=1}^{T} (f(h_{I_t(1)}(\tilde{\mathbf{S}}^t)) u_{I_t(1),t}^{UCB} - f(h_{I_t(1)}(\tilde{\mathbf{S}}^t)) u_{I_t(1)}) \left( 1 + \sum_{r=1}^{N} g^r \right) 1(E_t) \right] +$$

4

$$E_\pi \left[ \sum_{t=1}^{T} ((g-q)f(h_{I_t(1)}(\tilde{\mathbf{S}}^t))u_{I_t(1)} + q)(f(h_{I_t(2)}(\tilde{\mathbf{S}}^t))u_{I_t(2),t}^{UCB} - f(h_{I_t(2)}(\tilde{\mathbf{S}}^t))u_{I_t(2)}) \left(1 + \sum_{r=1}^{N} g^r\right) 1(E_t) \right]$$

$$+ \cdots + E_\pi \left[ \sum_{t=1}^{T} \prod_{k=1}^{|\tilde{\mathbf{S}}^t|-1} ((g-q)f(h_{I_t(k)}(\tilde{\mathbf{S}}^t))u_{I_t(k)} + q)(f(h_{I_t(k)}(\tilde{\mathbf{S}}^t))u_{I_t(k),t}^{UCB} - f(h_{I_t(k)}(\tilde{\mathbf{S}}^t))u_{I_t(k)}) \right.$$
$$\left. \left(1 + \sum_{r=1}^{N} g^r\right) 1(E_t) \right] \tag{0.2}$$

$$\leq \frac{1-g^N}{1-g} \sum_{t=1}^{T} E_\pi \left[ \sum_{i=1}^{|\tilde{\mathbf{S}}^t|-1} \prod_{k=1}^{i-1} ((g-q)f(h_{I_t(k)}(\tilde{\mathbf{S}}^t))u_{I_t(k)} + q)(f(h_{I_t(k)}(\tilde{\mathbf{S}}^t))u_{I_t(k),t}^{UCB} - f(h_{I_t(k)}(\tilde{\mathbf{S}}^t))\right.$$
$$\left. u_{I_t(k)})1(E_t) \right]$$

$$\leq \frac{1-g^N}{1-g} E_\pi \left[ \sum_{t=1}^{T} \sum_{j=1}^{N} \kappa_{j,t}(u_{j,t}^{UCB} - u_{j,t})1(E_t) \right] \leq \frac{\sqrt{2}(1-g^N)}{1-g} E_\pi \left[ \sum_{t=1}^{T} \sum_{j=1}^{N} \kappa_{j,t} \sqrt{\frac{\log t}{T_j(t)}} \right], \tag{0.3}$$

where $\kappa_{j,t} = \prod_{k=1}^{l_t(j)-1}((g-q)f(h_{I_t(k)}(\tilde{\mathbf{S}}^t))u_{I_t(k)} + q)$ and $l_t(i)$ specifies the order of item $i$ in $\tilde{\mathbf{S}}^t$. $\kappa_{j,t}$ denotes the probability of observing item $j$ in sequence $\tilde{\mathbf{S}}^t$.

Recall that $T_j(t)$ is defined as the total number of users who observe item corresponding to $f(j)$ before time $t$. Conditioned on $\tilde{\mathbf{S}}^t$, $T_j(t+1) = T_j(t) + 1$ with probability at least $\kappa_{j,t}$. Now we provide an upper bound on the term $E_\pi \left[ \sum_{t=1}^{T} \kappa_{j,t}\sqrt{1/T_j(t)} \right]$. Conditioned on $T_j(t_n) = n$ and $T_j(t_n - 1) = n - 1$, we will bound the summation when $T_j$ equals $n$. Define $T'$ as the stopping time when $T_j(t)$ changes from $n$ to $n+1$, and we say it is a success if it jumps to $n+1$. Then $T' = k$ with probability $\prod_{i=1}^{k-1}(1 - \kappa_{j,t_n+i-1})\kappa_{j,t_n+k-1}$, which is a joint probability that it fails in the first $k-1$ times, and jumps to $n+1$ on the $k^{th}$ time. Then the summation when $T_j$ stays on $n$ is $\frac{1}{\sqrt{n}} E_\pi \left[ \sum_{t=t_n}^{t_n+T'-1} \kappa_{j,t} \right]$. For ease of notation, let $d_t = \kappa_{j,t+t_n-1}$, so $T' = k$ with probability $\prod_{i=1}^{k-1}(1-d_i)d_k$. It implies that

$$E_\pi \left[ \sum_{t=t_n}^{t_n+T'-1} \kappa_{j,t} \right] = E_\pi \left[ \sum_{t=1}^{T'} d_t \right] = E \left[ \sum_{k=1}^{\infty} P(T'=k) \left( \sum_{s=1}^{k} d_s \right) \right]$$
$$= E \left[ \sum_{k=1}^{\infty} P(T' \geq k) d_k \right] = E \left[ \sum_{k=1}^{\infty} \prod_{i=1}^{k-1}(1-d_i)d_k \right] = 1,$$

where the last equality holds because it is a probability sum. It implies that

$$\frac{1-g^N}{1-g} E_\pi \left[ \sum_{i \in [N]} \sum_{t=1}^{T} \kappa_{i,t}(u_{i,t}^{UCB} - u_i) 1(E_t) \right]$$
$$\leq \frac{1-g^N}{1-g} \sum_{i \in [N]} \sqrt{2 \log T} E_\pi \left[ \sum_{n=1}^{T_i(T)} \sqrt{1/n} \right]$$

5

$$\leq \frac{1-g^N}{1-g}\sqrt{2\log T}\sum_{i\in[N]} E_\pi[\sqrt{T_i(T)}]$$

$$\leq \frac{1-g^N}{1-g}\sqrt{2\log T}\sum_{i\in[N]} \sqrt{E_\pi[T_i(T)]}, \qquad (0.4)$$

where the last inequality holds because of Jensen's inequality. Since $\sum_{i\in[N]} E_\pi[T_i(T)] \leq \frac{1-g^N}{1-g}T$, we have

$$\frac{1-g^N}{1-g}\sqrt{2\log T}\sum_{i\in[N]} \sqrt{E[T_i(T)]} \leq C\left(\frac{1-g^N}{1-g}\right)^{3/2}\sqrt{NT\log T}.$$

According to Lemma 4.1, we have

$$E_\pi\left[\sum_{t=1}^T \left(R(\mathbf{S}^*,\mathbf{u}) - R(\tilde{\mathbf{S}}^t,\mathbf{u})\right)\left(\sum_{i=1}^N 1(A_{i,t}^c)\right)\right]$$

$$\leq R(\mathbf{S}^*,\mathbf{u})E_\pi\left[\sum_{i\in[N]}\sum_{t=1}^T \frac{2}{t^4}\right] \leq C_4 R^* N.$$

Therefore

$$E_\pi\left[\sum_{t=1}^T R(\mathbf{S}^*,\mathbf{u}) - R(\tilde{\mathbf{S}}^t,\mathbf{u})\right]$$

$$\leq E_\pi\left[\sum_{t=1}^T \left(R(\mathbf{S}^*,\mathbf{u}) - R(\tilde{\mathbf{S}}^t,\mathbf{u})\right) 1\left(\bigcap_{i=1}^N A_{i,t}\right)\right]$$

$$+ E_\pi\left[\sum_{t=1}^T \left(R(\mathbf{S}^*,\mathbf{u}) - R(\tilde{\mathbf{S}}^t,\mathbf{u})\right)\left(\sum_{i=1}^N 1(A_{i,t}^c)\right)\right]$$

$$\leq C\left(\frac{1-g^N}{1-g}\right)^{3/2}\sqrt{NT\log T}$$

for some constant $C$.

∎

**Corollary 4.2** *When $g = q = 1$, the regret can be bounded by*

$$Regret_\pi(T) \leq CN^2\sqrt{T\log T},$$

*for some constant $C$. Moreover, if at most $L$ items will be recommended, then the regret can be bounded by $CL^{3/2}\sqrt{NT\log T}$.*

**Proof of Corollary 4.2.** If $g = q = 1$, most of the proofs follow from Theorem 4.1. In Equation (0.2), the term $(1 + \sum_{r=1}^N g^r)$ will be replaced by $N$ so Equation (0.3) will be

$C_1 N E_\pi \left[ \sum_{t=1}^T \sum_{j=1}^N \kappa_{j,t} \sqrt{\frac{\log t}{T_j(t)}} \right]$. Similarly, the coefficient in Equation (0.4) will also be $N$. Since $\sum_{i \in [N]} T_i(T) \leq NT$ with probability 1, the leading term in the regret can be bounded from above by

$$CN\sqrt{2\log T} \sum_{i \in [N]} \sqrt{E[T_i(T)]} \leq CN^2 \sqrt{T \log T}.$$

Similarly, if at most $L$ items can be recommended (capacity constraint), then in Equation (0.2), the term $(1 + \sum_{r=1}^N g^r)$ will be replaced by $L$. Since $\sum_{i \in [N]} T_i(T) \leq LT$ with probability 1, the leading term in the regret can be bounded from above by

$$CL\sqrt{2\log T} \sum_{i \in [N]} \sqrt{E[T_i(T)]} \leq CL^{3/2} \sqrt{NT \log T}.$$

Other parts of proof follow similarly. ∎

Define $V_i^t = \frac{1}{\hat{T}_i} \sum_j \sum_{t \in \mathcal{T}_{ij}(t)} \frac{z_{ij}^t}{u_j}$ and $\hat{T}_i(t) = \sum_j T_{ij}(t)$. We first have the following lemma to bound the difference between $V_i^t$ and $f_i$.

**Lemma 5.1** *For any $t$, we have $P\left(\left|V_i^t - f(i)\right| \geq \sqrt{\frac{\log t}{\hat{T}_i(t)}}\right) \leq \frac{2}{t^2}$.*

**Proof of Lemma 5.1.** We first note that for any $j$ and $t \in \mathcal{T}_{ij}$,

$$E[z_{ij}^t / u_j] = f(i).$$

Then we have

$$\begin{aligned}
&P(V_i^t - E[V_i^t] \geq u) \\
&= P(e^{s(V_i^t - E[V_i^t])} \geq e^{su}) \leq e^{-su} E[e^{s(V_i^t - E[V_i^t])}] \\
&= e^{-su} E\left[\exp\left(\frac{s}{\hat{T}_i(t)} \left(\sum_j \frac{1}{u_j} \sum_{t \in \mathcal{T}_{ij}} (z_t - f(i)u_j)\right)\right)\right] \\
&= e^{-su} \prod_{j=1}^k E\left[\exp\left(\frac{s}{\hat{T}_i(t) u_j} \sum_{t \in \mathcal{T}_{ij}} (z_t - f(i)u_j)\right)\right] \\
&= e^{-su} \prod_{j=1}^k \prod_{t \in \mathcal{T}_{ij}} E\left[\exp\left(\frac{s}{\hat{T}_i(t)} \left(\frac{z_t}{u_j} - f(i)\right)\right)\right] \\
&\leq e^{-su} \exp\left(\hat{T}_i(t) \frac{s^2}{8}\right) = \exp\left(-su + \frac{s^2 \hat{T}_i(t)}{8}\right).
\end{aligned}$$

Let $s = 4u\hat{T}_i(t)$, then we have

$$\exp\left(-su + \frac{s^2 \hat{T}_i(t)}{8}\right) = \exp\left(-2u^2 \hat{T}_i(t)\right)$$

7

Let $u = \sqrt{\log t / \hat{T}_i(t)}$. Then we have

$$\exp\left(-2u^2 \hat{T}_i(t)\right) = \frac{1}{t^2}.$$

∎

**Lemma 5.2** *For any $1 \leq i \leq M$, we have*

$$\sum_{t=1}^{T} P\left(f_t^{UCB}(i) - 2\Delta_i^t \leq f(i) \leq f_t^{UCB}(i)\right) \leq CN,$$

*for some constant $C$.*

**Proof.** First we note that

$P\left(f_t^{UCB}(i) - 2\Delta_i^t \leq f(i) \leq f_t^{UCB}(i)\right)$

$= P\left(\left|\hat{f}_{i,t} - f(i)\right| \geq \sum_j \sum_{r \in \hat{\mathcal{T}}_{ij}} \frac{z_{ij}^r}{\hat{T}_i(t)} \frac{1}{\hat{u}_{j,t}^2} \left(1 - \frac{1}{\hat{u}_{j,t}} \sqrt{\frac{\log t}{T_{0j}(t)}}\right) \mathbf{1}\left(\hat{u}_{j,t} \geq \sqrt{\frac{\log t}{T_{0j}(t)}}\right) \sqrt{\frac{\log t}{T_{0j}(t)}} + \sqrt{\frac{\log t}{T_i(t)}}\right)$

$\leq P\left(\left|\frac{1}{\hat{T}_i(t)} \sum_j \sum_{r \in \mathcal{T}_{ij}} \frac{z_{ij}^r}{\hat{u}_{j,t}} - \frac{1}{\hat{T}_i(t)} \sum_j \sum_{r \in \mathcal{T}_{ij}} \frac{z_{ij}^r}{u_j} + \frac{1}{\hat{T}_i(t)} \sum_j \sum_{r \in \mathcal{T}_{ij}} \frac{z_{ij}^r}{u_j} - f(i)\right| \geq \right.$

$\left. \sum_{r \in \mathcal{T}_{ij}} \frac{z_{ij}^t}{\hat{T}_i(t)} \frac{1}{\hat{u}_{j,t}^2} \left(1 - \frac{1}{\hat{u}_{j,t}} \sqrt{\frac{\log t}{T_{0j}(t)}}\right) \mathbf{1}\left(\hat{u}_{j,t} \geq \sqrt{\frac{\log t}{T_{0j}(t)}}\right) \sqrt{\frac{\log t}{T_{0j}(t)}} + \sqrt{\frac{\log t}{T_i(t)}}\right)$

$\leq P\left(\left|\frac{1}{\hat{T}_i(t)} \sum_j \sum_{r \in \mathcal{T}_{ij}} z_{ij}^r \left(\frac{1}{\hat{u}_{j,t}} - \frac{1}{u_j}\right)\right| \geq \sum_j \sum_{r \in \hat{\mathcal{T}}_{ij}} \frac{z_{ij}^r}{\hat{T}_i(t)} \frac{1}{\hat{u}_{j,t}^2} \left(1 - \frac{1}{\hat{u}_{j,t}} \sqrt{\frac{\log t}{T_{0j}(t)}}\right)\right.$

$\left. \mathbf{1}\left(\hat{u}_{j,t} \geq \sqrt{\frac{\log t}{T_{0j}(t)}}\right) \sqrt{\frac{\log t}{T_{0j}(t)}}\right)$ (0.5)

$+ P\left(\left|\frac{1}{\hat{T}_i(t)} \sum_j \sum_{r \in \mathcal{T}_{ij}} \frac{z_{ij}^r}{u_j} - f(i)\right| \geq \sqrt{\frac{\log t}{T_i(t)}}\right)$ (0.6)

For Equation (0.5), we can bound it above by

$\sum_j P\left(\sum_{r \in \mathcal{T}_{ij}} \frac{z_{ij}^r}{\hat{T}_i(t)} \frac{|u_j - \hat{u}_{j,t}|}{u_j \hat{u}_{j,t}} \geq \sum_{r \in \mathcal{T}_{ij}} \frac{z_{ij}^r}{\hat{T}_i(t)} \frac{1}{\hat{u}_{j,t}^2} \left(1 - \frac{1}{\hat{u}_{j,t}} \sqrt{\frac{\log t}{T_{0j}(t)}}\right) \mathbf{1}\left(\hat{u}_{j,t} \geq \sqrt{\frac{\log t}{T_{0j}(t)}}\right) \sqrt{\frac{\log t}{T_{0j}(t)}}\right)$

$\leq \sum_j E\left[\mathbf{1}\left(|u_j - \hat{u}_{j,t}| \geq \frac{u_j}{\hat{u}_{j,t}} \left(1 - \frac{1}{\hat{u}_{j,t}} \sqrt{\frac{\log t}{T_{0j}(t)}}\right) \mathbf{1}\left(\hat{u}_{j,t} \geq \sqrt{\frac{\log t}{T_{0j}(t)}}\right) \sqrt{\frac{\log t}{T_{0j}(t)}}\right)\right.$

$\left. \mathbf{1}\left(\frac{u_j}{\hat{u}_{j,t}} \geq 1 - \frac{1}{\hat{u}_{j,t}} \sqrt{\frac{\log t}{T_{0j}(t)}} \geq 0\right)\right]$

8

$$+ E\left[1\left(\frac{u_j}{\hat{u}_{j,t}} < 1 - \frac{1}{\hat{u}_{j,t}}\sqrt{\frac{\log t}{T_{0j}(t)}}\right)\right] + E\left[\sum_{t \in \mathcal{T}_{ij}} 1\left(\hat{u}_{j,t} < \sqrt{\frac{\log t}{T_{0j}(t)}}\right)\right]$$

$$\leq \sum_j P\left(|u_j - \hat{u}_{j,t}| \geq \sqrt{\frac{\log t}{T_{0j}(t)}}\right) + P\left(|u_j - \hat{u}_{j,t}| \geq \sqrt{\frac{\log t}{T_{0j}(t)}}\right) + P\left(\hat{u}_{j,t} < \sqrt{\frac{\log t}{T_{0j}(t)}}\right)$$

$$\leq \sum_j \frac{2}{t^2} + E\left[1\left(\hat{u}_{j,t} < \sqrt{\frac{\log t}{T_{0j}(t)}}\right) 1\left(\sqrt{\frac{\log t}{T_{0j}(t)}} < \frac{u_j}{2}\right)\right] + P\left(\sqrt{\frac{\log t}{T_{0j}(t)}} < \frac{u_j}{2}\right)$$

$$\leq \sum_j \frac{2}{t^2} + \exp\left(-\frac{1}{2}T_{0j}(t)u_j^2\right) + P\left(\sqrt{\frac{\log t}{T_{0j}(t)}} < \frac{u_j}{2}\right).$$

For Equation (0.6), using Lemma 5.1, we have

$$P\left(\left|\frac{1}{\hat{T}_i(t)} \sum_j \sum_{r \in \mathcal{T}_{ij}} \frac{z_{ij}^r}{u_j} - f(i)\right| \geq \sqrt{\frac{\ln t}{T_i(t)}}\right) \leq \frac{2}{t^2}.$$

It implies that

$$\sum_{t=1}^T P\left(f_t^{UCB}(i) - 2\Delta_i^t \leq f(i) \leq f_t^{UCB}(i)\right)$$

$$\leq \sum_{t=1}^T \left(\sum_j \frac{4}{t^2} + \exp\left(-\frac{1}{2}T_{0j}(t)u_j^2\right) + P\left(\sqrt{\frac{\log t}{T_{0j}(t)}} < \frac{u_j}{2}\right)\right) \leq CN$$

for some constant $C$. ∎

**Theorem 5.4** *Given all parameters $\theta$, the regret of Algorithm 3 during time $T$ is bounded by*

$$Regret_\pi(T) \leq C\left(\frac{1-g^N}{1-g}\right)^2 \sqrt{NT^{4/3}\log T}$$

*for some constant $C$.*

**Proof of Theorem 5.4.**

Define $\mathcal{F}_t$ as the filtration associated with the policy $\pi$ up to time $t$. Let

$$A_{i,t} = \left\{u_{i,t}^{UCB} - \sqrt{8\frac{\log t}{T_i(t)}} < u_i < u_{i,t}^{UCB}\right\}, \quad B_{i,t} = \left\{f_t^{UCB}(i) - 2\Delta_i^t < f(i) < f_t^{UCB}(i)\right\},$$

Define

$$E_t = \bigcap_{j=1}^N (A_{i,t}) \bigcap_{i=1}^M (B_{i,t}).$$

Let $\tilde{f}_t^{UCB}(i) = \min_{j \leq i} f_t^{UCB}(j)$, $\tilde{B}_{i,t} = \left\{ \tilde{f}_t^{UCB}(i) - 2\Delta_i^t < f(i) < \tilde{f}_t^{UCB}(i) \right\}$ and

$$\tilde{E}_t = \bigcap_{j=1}^{N} (A_{i,t}) \bigcap_{i=1}^{M} (\tilde{B}_{i,t}).$$

In Algorithm 3, in each round, we use parameter $\tilde{f}_t^{UCB}(i)$. Note that if $E_t$ holds, $\tilde{E}_t$ also holds because for any $i \in [M]$, $f(i) \leq f(j) < f_t^{UCB}(j)$ for all $j \leq i$ so it implies that $f(i) < \tilde{f}_t^{UCB}(i)$. For the other direction, since $\tilde{f}_t^{UCB}(i) \leq f_t^{UCB}(i)$, we have $\tilde{f}_t^{UCB}(i) - 2\Delta_i^t < f(i)$. We say $t$ is in the exploration period, denoted as $t \in \mathcal{E}$, if there exists $j \in [N]$ such that $T_{0j} < \alpha t^{2/3}$. Define $\tilde{\mathbf{S}}^t$ as the optimal sequence when the corresponding parameters are

$$\theta_t^{UCB} = (\mathbf{u}_t^{UCB}, f_t^{UCB}).$$

Let $I_t(\cdot)$ be the index function of messages in $\tilde{\mathbf{S}}^t$. On event $E_t$, according to Lemma ??, we have

$$E_\pi \left[ R(\tilde{\mathbf{S}}^t, \theta_t) \right] \leq E_\pi [R(\mathbf{S}^*, \theta_t)] \leq E_\pi [R(\mathbf{S}^*, \theta_t^{UCB})] \leq E_\pi [R(\tilde{\mathbf{S}}^t, \theta_t^{UCB})],$$

which implies that on event $E_t$, we have

$$E_\pi \left[ \sum_{t=1}^{T} \left( R(\mathbf{S}^*, \theta) - R(\tilde{\mathbf{S}}^t, \theta) \right) \mathbf{1}(E_t) \mathbf{1}(t \in \mathcal{E}) \right]$$

$$\leq E_\pi \left[ \sum_{t=1}^{T} E_\pi \left[ \left( R(\tilde{\mathbf{S}}^t, \theta_t^{UCB}) - R(\tilde{\mathbf{S}}^t, \theta) \right) \mathbf{1}\left(\tilde{E}_t\right) \mathbf{1}(t \in \mathcal{E}) \Big| \mathcal{F}_{t-1} \right] \right]$$

$$= E_\pi \left[ \sum_{t=1}^{T} \sum_{i=1}^{|\tilde{\mathbf{S}}^t|} \prod_{k=1}^{i-1} \left( (g-q) \tilde{f}_t^{UCB}(h_{I_t(k)}(\tilde{\mathbf{S}}^t)) u_{I_t(k),t}^{UCB} + q \right) \tilde{f}_t^{UCB}(h_{I_t(i)}(\tilde{\mathbf{S}}^t)) u_{I_t(i),t}^{UCB} \mathbf{1}(\tilde{E}_t) \right]$$

$$- E_\pi \left[ \sum_{t=1}^{T} \sum_{i=1}^{|\tilde{\mathbf{S}}^t|} \prod_{k=1}^{i-1} \left( (g-q) f(h_{I_t(k)}(\tilde{\mathbf{S}}^t)) u_{I_t(k)} + q \right) f(h_{I_t(i)}(\tilde{\mathbf{S}}^t)) u_{I_t(i)} \mathbf{1}(\tilde{E}_t) \right]$$

$$\leq E_\pi \left[ \sum_{t=1}^{T} \sum_{i=1}^{|\tilde{\mathbf{S}}^t|} \prod_{k=1}^{i-1} \left( (g-q) \tilde{f}_t^{UCB}(h_{I_t(k)}(\tilde{\mathbf{S}}^t)) u_{I_t(k),t}^{UCB} + q \right) \tilde{f}_t^{UCB}(h_{I_t(i)}(\tilde{\mathbf{S}}^t)) u_{I_t(i),t}^{UCB} \mathbf{1}(\tilde{E}_t) \right]$$

$$- E_\pi \left[ \sum_{t=1}^{T} \sum_{i=1}^{|\tilde{\mathbf{S}}^t|} \prod_{k=1}^{i-1} \left( (g-q) f(h_{I_t(k)}(\tilde{\mathbf{S}}^t)) u_{I_t(k)} + q \right) f(h_{I_t(i)}(\tilde{\mathbf{S}}^t)) u_{I_t(i)} \mathbf{1}(\tilde{E}_t) \right] \quad (0.7)$$

To analyze Equation 0.7, we have

$$E_\pi \left[ \sum_{t=1}^{T} \sum_{i=1}^{|\tilde{\mathbf{S}}^t|} \prod_{k=1}^{i-1} \left( (g-q) \tilde{f}_t^{UCB}(h_{I_t(k)}(\tilde{\mathbf{S}}^t)) u_{I(k),t}^{UCB} + q \right) \tilde{f}_t^{UCB}(h_{I_t(i)}(\tilde{\mathbf{S}}^t)) u_{I(i),t}^{UCB} \mathbf{1}(\tilde{E}_t) \right]$$

$$- E_\pi \left[ \sum_{t=1}^{T} \sum_{i=1}^{|\tilde{\mathbf{S}}^t|} \prod_{k=1}^{i-1} \left( (g-q) f(h_{I_t(k)}(\tilde{\mathbf{S}}^t)) u_{I_t(k)} + q \right) f(h_{I_t(i)}(\tilde{\mathbf{S}}^t)) u_{I_t(i)} \mathbf{1}(\tilde{E}_t) \right]$$

10

$$\leq E_\pi \left[ \sum_{t=1}^{T} \sum_{i=1}^{|\tilde{\mathbf{S}}^t|} \prod_{k=1}^{i-1} \left( (g-q)\tilde{f}_t^{UCB}(h_{I_t(k)}(\tilde{\mathbf{S}}^t))u_{I_t(k),t}^{UCB} + q \right) \tilde{f}_t^{UCB}(h_{I_t(i)}(\tilde{\mathbf{S}}^t))u_{I(i),t}^{UCB} 1(\tilde{E}_t) \right]$$

$$- E_\pi \left[ \sum_{t=1}^{T} \sum_{i=1}^{|\tilde{\mathbf{S}}^t|} \prod_{k=1}^{i-1} \left( (g-q)f(h_{I_t(k)}(\tilde{\mathbf{S}}^t))u_{I(k)}1(k=1) + (g-q)\tilde{f}_t^{UCB}(h_{I_t(k)}(\tilde{\mathbf{S}}^t))u_{I_t(k),t}^{UCB}1(k>1) + q \right) \right.$$
$$\left. \times \left( f(h_{I_t(i)}(\tilde{\mathbf{S}}^t))u_{I_t(i)}1(i=1) + \tilde{f}_t^{UCB}(h_{I_t(i)}(\tilde{\mathbf{S}}^t))u_{I_t(i),t}^{UCB}1(i>1) \right) 1(\tilde{E}_t) \right]$$

$$+ E_\pi \left[ \sum_{t=1}^{T} \sum_{i=1}^{|\tilde{\mathbf{S}}^t|} \prod_{k=1}^{i-1} \left( (g-q)f(h_{I_t(k)}(\tilde{\mathbf{S}}^t))u_{I(k)}1(k=1) + (g-q)\tilde{f}_t^{UCB}(h_{I_t(k)}(\tilde{\mathbf{S}}^t))u_{I_t(k),t}^{UCB}1(k>1) + q \right) \right.$$
$$\left. \times \left( f(h_{I_t(i)}(\tilde{\mathbf{S}}^t))u_{I_t(i)}1(i=1) + \tilde{f}_t^{UCB}(h_{I_t(i)}(\tilde{\mathbf{S}}^t))u_{I_t(i),t}^{UCB}1(i>1) \right) 1(\tilde{E}_t) \right]$$

$$- E_\pi \left[ \sum_{t=1}^{T} \sum_{i=1}^{|\tilde{\mathbf{S}}^t|} \prod_{k=1}^{i-1} \left( (g-q)f(h_{I_t(k)}(\tilde{\mathbf{S}}^t))u_{I(k)}1(i \leq 2) + (g-q)\tilde{f}_t^{UCB}(h_{I_t(k)}(\tilde{\mathbf{S}}^t))u_{I_t(k),t}^{UCB}1(i>2) + q \right) \right.$$
$$\left. \times \left( f(h_{I_t(i)}(\tilde{\mathbf{S}}^t))u_{I_t(i)}1(i=2) + \tilde{f}_t^{UCB}(h_{I_t(i)}(\tilde{\mathbf{S}}^t))u_{I_t(i),t}^{UCB}1(i>2) \right) 1(\tilde{E}_t) \right]$$

$$+ E_\pi \left[ \sum_{t=1}^{T} \sum_{i=1}^{|\tilde{\mathbf{S}}^t|} \prod_{k=1}^{i-1} \left( (g-q)f(h_{I_t(k)}(\tilde{\mathbf{S}}^t))u_{I(k)}1(i \leq 2) + (g-q)\tilde{f}_t^{UCB}(h_{I_t(k)}(\tilde{\mathbf{S}}^t))u_{I_t(k),t}^{UCB}1(i>2) + q \right) \right.$$
$$\left. \times \left( f(h_{I_t(i)}(\tilde{\mathbf{S}}^t))u_{I_t(i)}1(i=2) + \tilde{f}_t^{UCB}(h_{I_t(i)}(\tilde{\mathbf{S}}^t))u_{I_t(i),t}^{UCB}1(i>2) \right) 1(\tilde{E}_t) \right] + \cdots$$

$$- E_\pi \left[ \sum_{t=1}^{T} \sum_{i=1}^{|\tilde{\mathbf{S}}^t|} \prod_{k=1}^{i-1} \left( (g-q)f(h_{I_t(k)}(\tilde{\mathbf{S}}^t))u_{I_t(k)} + q \right) f(h_{I_t(i)}(\tilde{\mathbf{S}}^t))u_{I_t(i)} 1(\tilde{E}_t) \right]$$

$$\leq E_\pi \left[ \sum_{t=1}^{T} (\tilde{f}_t^{UCB}(h_{I_t(1)}) - f(h_{I_t(1)})) \left( 1 + \sum_{r=1}^{N} g^r \right) 1(\tilde{E}_t) \right]$$

$$+ E_\pi \left[ \sum_{t=1}^{T} ((g-q)f(h_{I_t(1)})u_{I_t(1)} + q)(\tilde{f}_t^{UCB}(h_{I_t(2)}) - f(h_{I_t(2)})) \left( 1 + \sum_{r=1}^{N} g^r \right) \right] + \cdots$$

$$+ E_\pi \left[ \sum_{t=1}^{T} \prod_{k=1}^{|\tilde{\mathbf{S}}^t|-1} ((g-q)f(h_{I_t(k)})u_{I_t(k)} + q)(\tilde{f}_t^{UCB}(h_{I_t(k)}) - f(h_{I_t(k)})) \left( 1 + \sum_{r=1}^{N} g^r \right) \right]$$

$$\leq E_\pi \left[ \sum_{t=1}^{T} (\tilde{f}_t^{UCB}(h_{I_t(1)}(\tilde{\mathbf{S}}^t))u_{I_t(1),t}^{UCB} - f(h_{I_t(1)}(\tilde{\mathbf{S}}^t))u_{I_t(1)}) \left( 1 + \sum_{r=1}^{N} g^r \right) 1(\tilde{E}_t) \right]$$

$$+ E_\pi \left[ \sum_{t=1}^{T} ((g-q)f(h_{I_t(1)}(\tilde{\mathbf{S}}^t))u_{I_t(1)} + q) \right.$$
$$\left. (\tilde{f}_t^{UCB}(h_{I_t(2)}(\tilde{\mathbf{S}}^t))u_{I_t(2),t}^{UCB} - f(h_{I_t(2)}(\tilde{\mathbf{S}}^t))u_{I_t(2)}) \left( 1 + \sum_{r=1}^{N} g^r \right) 1(\tilde{E}_t) \right] + \cdots$$

$$+ E_\pi \left[ \sum_{t=1}^{T} \prod_{k=1}^{|\tilde{S}^t|-1} ((g-q)f(h_{I_t(k)}(\tilde{S}^t))u_{I_t(k)} + q)(\tilde{f}_t^{UCB}(h_{I_t(k)}(\tilde{S}^t))u_{I_t(k),t}^{UCB} - f(h_{I_t(k)}(\tilde{S}^t))u_{I_t(k)}) \right.$$

$$\left. \left(1 + \sum_{r=1}^{N} g^r\right) 1(\tilde{E}_t) \right]$$

$$\leq \frac{1-g^N}{1-g} \sum_{t=1}^{T} E_\pi \left[ \sum_{i=1}^{|\tilde{S}^t|-1} \prod_{k=1}^{i-1}((g-q)f(h_{I_t(k)}(\tilde{S}^t))u_{I_t(k)} + q)(\tilde{f}_t^{UCB}(h_{I_t(k)}(\tilde{S}^t))u_{I_t(k),t}^{UCB} \right.$$

$$\left. - \tilde{f}_t^{UCB}(h_{I_t(k)}(\tilde{S}^t))u_{I_t(k)} + \tilde{f}_t^{UCB}(h_{I_t(k)}(\tilde{S}^t))u_{I_t(k)} - f(h_{I_t(k)}(\tilde{S}^t))u_{I_t(k)})1(\tilde{E}_t) \right]$$

$$\leq \frac{1-g^N}{1-g} \left( E_\pi \left[ \sum_{t=1}^{T} \sum_{j=1}^{N} \kappa_{j,t}(u_{j,t}^{UCB} - u_j)1(\tilde{E}_t) \right] + E_\pi \left[ \sum_{t=1}^{T} \sum_{i=1}^{M} \tilde{\kappa}_{i,t}(\tilde{f}_t^{UCB}(i) - f(i))1(\tilde{E}_t) \right] \right),$$

where $\kappa_{j,t} = \prod_{k=1}^{l_t(j)-1}((g-q)f(h_{I_t(k)}(\tilde{S}^t))u_{I_t(k)} + q)$ denoting the probability of observing item $j$ in sequence $\tilde{S}^t$ and $\tilde{\kappa}_{i,t} = \sum_{\tilde{l}_t(i)} \prod_{k=1}^{\tilde{l}_t(i)-1}((g-q)f(h_{I_t(k)}(\tilde{S}^t))u_{I_t(k)} + q)$ denoting the probability of observing item arranged as the $i^{th}$ within the category. Since

$$(u_{j,t}^{UCB} - u_j)1(\tilde{E}_t) \leq \sqrt{8\frac{\log t}{T_{0i}(t)}} \qquad T_{0i}(t) \geq \alpha t^{2/3}/\sqrt{N}, \quad \text{and} \quad E\left[\sum_{j=1}^{N} \kappa_{j,t}\right] \leq \frac{1-g^N}{1-g},$$

we have

$$\frac{1-g^N}{1-g} E_\pi \left[ \sum_{t=1}^{T} \sum_{j=1}^{N} \kappa_{j,t}(u_{j,t}^{UCB} - u_j)1(\tilde{E}_t) \right] \leq \frac{1-g^N}{1-g} \sqrt{8\log T} E_\pi \left[ \sum_{t=1}^{T} \sum_{j=1}^{N} \kappa_{j,t} \frac{\sqrt{N}}{t^{1/3}} \right]$$

$$\leq C_1 \sqrt{8N\log T} \left(\frac{1-g^N}{1-g}\right)^2 E\left[\sum_{t=1}^{T} t^{-1/3}\right] \leq C_1 \left(\frac{1-g^N}{1-g}\right)^2 \sqrt{NT^{4/3}\log T}.$$

Recall that

$$\tilde{f}_t^{UCB}(i) - f(i)$$

$$\leq 2\Delta_i^t = 2\left(\sum_j \sum_{r \in \mathcal{T}_{ij}} \frac{z_{ij}^r}{\hat{T}_i(t)} \frac{1}{\hat{u}_{j,t}^2}\left(1 - \frac{1}{\hat{u}_{j,t}}\sqrt{\frac{\log t}{T_{0j}(t)}}\right) 1\left(\hat{u}_{j,t} \geq \sqrt{\frac{\log t}{T_{0j}(t)}}\right) \sqrt{\frac{\log t}{T_{0j}(t)}} + \sqrt{\frac{\log t}{\hat{T}_i(t)}}\right).$$

Therefore, we have

$$E_\pi \left[ \sum_{t=1}^{T} \sum_{i=1}^{M} \tilde{\kappa}_{i,t}(f_t^{UCB}(i) - f(i)) \right]$$

$$\leq 2E_\pi \left[ \sum_{t=1}^{T} \sum_{i=1}^{M} \tilde{\kappa}_{i,t} \left(\sum_j \sum_{r \in \mathcal{T}_{ij}} \frac{z_{ij}^r}{\hat{T}_i(t)} \frac{1}{\hat{u}_{j,t}^2}\left(1 - \frac{1}{\hat{u}_{j,t}}\sqrt{\frac{\log t}{T_{0j}(t)}}\right) 1\left(\hat{u}_{j,t} \geq \sqrt{\frac{\log t}{T_{0j}(t)}}\right) \sqrt{\frac{\log t}{T_{0j}(t)}} + \sqrt{\frac{\log t}{\hat{T}_i(t)}}\right) \right]$$

12

$$\leq \frac{2}{\mu_{min}^2} E_\pi \left[ \sum_{t=1}^{T} \sum_{i=1}^{M} \tilde{\kappa}_{i,t} \sqrt{\frac{\log t}{T_{0j}(t)}} \right] + 2 E_\pi \left[ \sum_{t=1}^{T} \sum_{i=1}^{M} \tilde{\kappa}_{i,t} \sqrt{\frac{\log t}{\hat{T}_i(t)}} \right]. \qquad (0.8)$$

Since
$$E_\pi \left[ \sum_{i=1}^{M} \tilde{\kappa}_{i,t} \right] \leq \frac{1-g^N}{1-g},$$

we have

$$E_\pi \left[ \sum_{t=1}^{T} \sum_{i=1}^{M} \tilde{\kappa}_{i,t} \sqrt{\frac{\log t}{T_{0j}(t)}} \right] \leq \frac{1-g^N}{1-g} E_\pi \left[ \sum_{t=1}^{T} \sqrt{\frac{\log t}{T_{0j}(t)}} \right]$$
$$\leq \frac{(1-g^N)\sqrt{\log T}}{1-g} E_\pi \left[ \sum_{t=1}^{T} \frac{1}{\sqrt{T_{0j}(t)}} \right] \leq C_2 \frac{1-g^N}{1-g} \sqrt{NT^{4/3} \log T}.$$

For the second part in Equation (0.8), we can use the similar argument in the proof in Theorem 4.1 to conclude that that

$$E_\pi \left[ \sum_{t=1}^{T} \sum_{i=1}^{M} \tilde{\kappa}_{i,t} \sqrt{\frac{\log t}{\hat{T}_i(t)}} \right] \leq \sum_{i=1}^{M} \sqrt{2 \log T} E_\pi \left[ \sum_{n=1}^{\hat{T}_i(T)} \sqrt{1/n} \right] \leq \sqrt{2 \log T} \sum_{i=1}^{M} E_\pi \left[ \sqrt{\hat{T}_i(T)} \right].$$

Since $\sum_{i=1}^{M} E[\hat{T}_i(T)] \leq \frac{1-g^N}{1-g} T$, we have

$$\sum_{i=1}^{M} E_\pi \left[ \sqrt{\hat{T}_i(T)} \right] \leq \sum_{i=1}^{M} \left[ \sqrt{E_\pi[\hat{T}_i(T)]} \right] \leq C_3 \sqrt{\frac{1-g^N}{1-g} MT},$$

which implies that

$$\frac{1-g^N}{1-g} E_\pi \left[ \sum_{t=1}^{T} \sum_{i=1}^{M} \tilde{\kappa}_{i,t} \sqrt{\frac{\log t}{\hat{T}_i(t)}} \right] \leq C_4 \left( \frac{1-g^N}{1-g} \right)^{3/2} \sqrt{MT \log T}.$$

According to Lemma 4.1, we have

$$E_\pi \left[ \sum_{t=1}^{T} \left( R(\mathbf{S}^*, \theta) - R(\tilde{\mathbf{S}}^t, \theta) \right) \left( \left( \sum_{j=1}^{N} \mathbf{1}(A_{j,t}^c) \right) + \left( \sum_{i=1}^{M} \mathbf{1}(B_{i,t}^c) \right) + \mathbf{1}(t \notin \mathcal{E}) \right) \right]$$
$$\leq R^* \left( E_\pi \left[ \sum_{i=1}^{N} \sum_{t=1}^{T} \frac{2}{t^4} + \sum_{i=1}^{M} N \right] + N\alpha T^{2/3}/\sqrt{N} \right) \leq C_5 R^* \sqrt{NT^{4/3}}.$$

where $R^*$ is the highest expected profit. Note that $R^* \leq \sum_{k=0}^{N} g^k \leq \frac{1-g^N}{1-g}$. It implies that

$$R^* \sqrt{NT^{4/3}} \leq \frac{1-g^N}{1-g} \sqrt{NT^{4/3}}.$$

13

Therefore

$$E_\pi \left[ \sum_{t=1}^T R(\mathbf{S}^*, \theta) - R(\tilde{\mathbf{S}}^t, \theta) \right]$$

$$\leq E_\pi \left[ \sum_{t=1}^T \left( R(\mathbf{S}^*, \theta) - R(\tilde{\mathbf{S}}^t, \theta) \right) 1(E_t) 1(t \in \mathcal{E}) \right]$$

$$+ E_\pi \left[ \sum_{t=1}^T (R(\mathbf{S}^*, \theta) - R(\tilde{\mathbf{S}}^t, \theta)) \left( \left( \sum_{j=1}^N 1(A_{j,t}^c) \right) + \left( \sum_{i=1}^M 1(B_{i,t}^c) \right) + 1(t \in \mathcal{E}) \right) \right]$$

$$\leq C_6 \left( \frac{1 - g^N}{1 - g} \right)^2 \sqrt{NT^{4/3} \log T}$$

for some constant $C_6$. ∎







\end{document}